%% file: main.tex
\documentclass[lettersize,journal]{IEEEtran}
\input{preamble}

\newcommand{\removeForSpace}[1]{}

\begin{document}

\title{
Exact Signed-Distance Control Barrier Functions via Minkowski Operations for Safe Navigation among Polytopes
}
\author{Yi-Hsuan Chen,~Shuo Liu,~\IEEEmembership{Student Member, IEEE},~Wei Xiao,~\IEEEmembership{Member, IEEE},\\~Calin Belta,~\IEEEmembership{Fellow, IEEE},~and Michael Otte,~\IEEEmembership{Senior Member, IEEE}
\thanks{Yi-Hsuan Chen and Michael Otte are with the Department of Aerospace Engineering, University of Maryland, College Park, MD 20742, USA (e-mail: yhchen91@umd.edu; otte@umd.edu).}
\thanks{Shuo Liu is with the Department of Mechanical Engineering, Boston University, Boston, MA 02215, USA (e-mail:  liushuo@bu.edu).}
\thanks{W. Xiao is with the School of Electrical and Electronic Engineering, Nanyang Technological University, and with M3S, SMART, Singapore (e-mail: wei.xiao@ntu.edu.sg, weixy@mit.edu).}
\thanks{Calin Belta is with the Departments of Electrical and Computer Engineering and Computer Science, University of Maryland, College Park, MD 20742, USA (e-mail: calin@umd.edu).}}
\maketitle

\begin{abstract}
Safely navigating polytopic environments while respecting the dynamics, control, and exact geometry of the underlying system is a challenge in robotics. Control barrier functions (CBFs) synthesize safe control policies by rendering the safe set forward invariant, but many existing CBF-based methods approximate polytopes using conservative smooth shapes, such as spheres or ellipsoids, to obtain explicit differentiable distance functions. In this article, we propose an exact Signed Distance Function (SDF) formulation for a {\it polytopic} robot and {\it polytopic} obstacles and integrate it with nonsmooth CBFs. Leveraging Minkowski operations, the proposed method computes the exact SDF via companion convex programs in both the collision-free (positive-sign) and in-collision (negative-sign) cases. Furthermore, by exploiting the convenient geometric properties of 2D Minkowski operations and the optimality conditions of the two companion convex programs, we derive a unified analytical expression for the gradient of the exact SDF via sensitivity analysis. The exact rotational gradient further reveals a previously masked class of local minima induced by the coupling between geometry and nonholonomic kinematics. We demonstrate the effectiveness of the proposed framework through a pure-translation case and three scenarios with unicycle models involving recovery from an unsafe initialization and single- and multiple-obstacle avoidance. Comparisons with baseline methods highlight how the proposed framework enables non-conservative maneuvers and safety recovery. A demonstration video can be found at \normalfont{\texttt{\url{https://youtu.be/D0zVswzyxaE}}}.
\end{abstract}


\section{Introduction}
Ensuring safe navigation while respecting dynamic and control constraints remains a fundamental challenge in robotics. This task becomes significantly more complex when the geometric constraints of both the ego agent and obstacles are also considered. While configuration obstacles (COs) \cite{Lozano.TC83} have long been used to reason about such geometric set interactions, traditional methods are typically decoupled from low-level control synthesis. Control Barrier Functions (CBFs) \cite{Ames.etal.ECC19} are widely used for collision avoidance, providing forward invariance guarantees for safe sets of control-affine systems.
Moreover, CBFs can be naturally combined with Control Lyapunov Functions (CLFs) through pointwise quadratic programs (QPs). The resulting CLF-CBF-QP framework provides a computationally efficient approach to reach-avoid problems and is well suited to real-time applications.
Further advances in CBF-based safety-critical control include high-relative-degree \cite{Nguyen.Sreenath.ACC16, Xiao.Belta.TAC22} and adaptive formulations \cite{liu2025auxiliary}.

However, because standard continuous-time CBFs \cite{Ames.etal.ECC19} require the barrier function to be continuously differentiable ($C^1$), many existing CBF formulations model robots and/or obstacles as smooth shapes, such as hyperspheres \cite{Chen.etal.TCST18} and ellipsoids \cite{Verginis.Dimarogonas.LCSS19,Funada.etal.TCST25}. These smooth approximations often have convenient explicit distance expressions but can be overly conservative for real-world objects, such as vehicles.
In contrast, polytopes can provide tighter abstractions, but their induced exact Signed Distance Functions (SDFs) are implicit (i.e., computed via optimization programs) and only piecewise smooth, e.g., non-differentiable in specific configurations due to having corners.
While discrete-time CBFs \cite{Thirugnanam.etal.ICRA22,Liu.etal.arXiv26,Gonzalez.etal.arXiv26} offer a potential workaround for non-differentiability, their safety guarantees generally apply only at discrete sampling instants and do not, by themselves, ensure continuous-time safety guarantees as considered here.

Our work is closely related to two lines of work depending on whether smooth shape approximations are used, which we will now survey in more depth. The first line of closely related work uses smooth approximations to ``soften'' the sharp corners to avoid nonsmoothness. For instance, polygons are approximated by polynomial functions fitted via
logistic regression in \cite{Peng.etal.ICRA23} or smoothed using log-sum-exp (LSE) relaxations \cite{Molnar.CCTA25,Usevitch.Sahleen.ACC25, Wu.etal.TCB25, Wei.etal.TCST26}. Although easily integrated with standard CBFs, these methods inevitably alter the true geometry, introducing a certain level of conservativeness.
We remark that boundary smoothing can render the one-sided Minimum-Distance Function (MDF) $C^1$ on the collision-free domain but cannot do so globally for the SDF. The latter extends into the set interior and remains nonsmooth on the skeleton, also known as the medial axis \cite{4M}, where each point has more than one closest point on the shape boundary.

The second line of closely related work handles geometry \textit{without approximation}. Early work applying CBFs for collision avoidance between polytopes \cite{Thirugnanam.etal.ACC22} uses the MDF to construct a nonsmooth CBF (NCBF) via a dualization technique
\cite{Zhang.OBCA.TCST20}, but computes the exact CBF derivative only for strongly convex sets and uses a conservative lower bound for weakly convex sets such as polytopes.
Likewise, another exact geometric approach using the SDF \cite{Singletary.etal.RAL22} relies on local linear approximations for the CBF gradients. 
\begin{table*}[t!]
    \centering
    \begin{minipage}[c]{0.69\textwidth}
        \centering
        \caption{Comparison of CBF-based collision avoidance methods for polytopic sets.}
        \label{table1}
        \vspace{-0.5em}
        
        \input{table_v1}

    \end{minipage}%
    \hfill 
    \begin{minipage}[l]{0.29\textwidth}
        \centering\
        \begin{xy}
            \xyimport(100,100){\includegraphics[width=\textwidth]{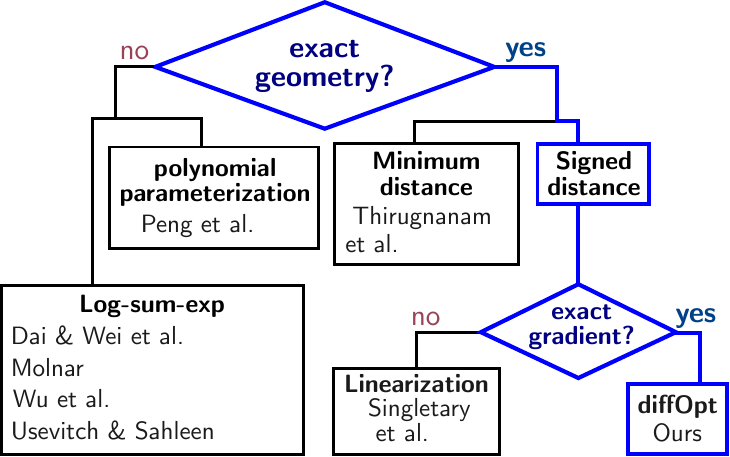}}
            ,(50,110)*{\text{}}
            ,(38.12,50.87)*{\text{\tiny\cite{Peng.etal.ICRA23}}}
            ,(32.54,26.31)*{\text{\tiny\cite{Dai.etal.RAL23,Wei.etal.TCST26}}}
            ,(15.32,19.26)*{\text{\tiny\cite{Molnar.CCTA25}}}
            ,(18.53,12.38)*{\text{\tiny\cite{Wu.etal.TCB25}}}
            ,(33.15,5.18)*{\text{\tiny\cite{Usevitch.Sahleen.ACC25}}}
            ,(62.11,46.54)*{\text{\tiny\cite{Thirugnanam.etal.ACC22,Thirugnanam.etal.arXiv23}}}
            ,(62.42,5.11)*{\text{\tiny\cite{Singletary.etal.RAL22}}}
        \end{xy}
        \vspace{-0.5cm}
        \figcaption{Taxonomy of existing methods based on the exactness of the SDF and its gradient.}
        \label{fig: taxonomy}
    \end{minipage}
\end{table*}

Our work differs from related work \cite{Peng.etal.ICRA23,Wu.etal.TCB25,Wei.etal.TCST26,Thirugnanam.etal.ACC22, Thirugnanam.etal.arXiv23, Singletary.etal.RAL22} in two important ways. First, we compute the {\it exact} SDF via convex optimization with COs in “Minkowski-difference-space”\footnote{We use the term  ``Minkowski-difference-space'' to distinguish it from the ``configuration space'' ($\mathcal{C}$-space) typically considered in motion planning. It is essentially the same as
the ``collision space'' defined in \cite{Montaut.etal.RSS22}. The more succinct ``Minkowski space'' has a specific different meaning in physics.}, denoted by \mdspace. Second, we compute the {\it exact} gradient of the {\it exact} SDF via differentiable optimization (diffOpt). Table \ref{table1} provides a conceptual comparison with existing works, and Figure~\ref{fig: taxonomy} further organizes them by exactness in 
geometry and gradient. While CBFs formulated in \mdspace~have previously been considered in \cite{Wu.etal.TCB25}, the corresponding optimization-free method relies on an LSE-smoothed SDF lower bound constructed from CO hyperplanes.

This article is a significant extension of our preliminary work \cite{Chen.etal.CDC25}, which leverages COs to construct CBFs from the exact SDF
between a {\it polytopic} robot and {\it polytopic} obstacles.
In addition to providing comprehensive technical details and proofs for the proposed Signed-Distance NCBF (SD-NCBF) framework, we derive exact, closed-form gradients with respect to the robot’s full SE(2) configuration, replacing the approximate rotational gradient used in \cite{Chen.etal.CDC25}. These exact gradients further reveal a class of geometric-kinematic local minima arising from the coupling between robot geometry and nonholonomic constraints.
The main contributions of this article are summarized as follows:
\begin{enumerate}
    \item \textbf{Exact SDF formulation and differentiability analysis:} 
    We formulate the exact SDF between polytopic sets in \mdspace~using companion convex programs for the safe and colliding branches, and characterize its differentiability with respect to the robot’s SE(2) configuration.
    
    \item \textbf{Unified exact-gradient framework:} 
    We establish the geometric equivalence and algebraic correspondence between the companion minimum-distance and penetration-depth convex programs in \mdspace, thereby providing a unified framework for exact-gradient computation across safe, contact, and penetration cases.

    \item \textbf{Closed-form SE(2) gradients and NCBF synthesis:} 
    We derive exact, closed-form spatial and rotational gradients of the exact SDF in SE(2) and integrate them into the SD-NCBF-QP framework.

    \item \textbf{Analysis of geometric-kinematic local minima:} 
    In conjunction with the previous contribution, 
    we identify a class of geometric-kinematic local minima
    induced by the coupling between robot geometry and nonholonomic constraints. We further characterize the robot shapes susceptible to this class of local minima.
\end{enumerate}

\section{Preliminaries}
In this section we present background on collision detection via CO (\ref{subsec: prelim_minkowski}), NCBFs (\ref{subsec: prelim_ncbf}), and diffOpt (\ref{subsec: prelim_diffopt}).
\subsection{Collision Detection via Minkowski Operations}\label{subsec: prelim_minkowski}
The workspace (\wspace) of a robot is the physical (2D or 3D Euclidean) environment in which it operates, consisting of the space where the robot can move and the obstacles it must avoid. The configuration space ($\mathcal{C}$-space) represents all possible positions, orientations, etc., with each point defining a unique configuration. Since detecting set-to-set collisions directly in $\mathcal{W}$~is nontrivial, a fundamental tool in computational geometry, the Minkowski sum (M-sum), is widely utilized to transform the problem into an equivalent point-to-set containment query. This geometric principle forms the foundation of the Gilbert-Johnson-Keerthi (GJK) algorithm\footnote{Note that GJK does not explicitly construct the CO; it queries extreme points of the CO via its support function, giving exact results under convexity.} for computing distance, and the Expanding Polytope Algorithm (EPA) \cite{epa} for computing penetration depth.

Motivated by this spatial mapping, we formally define the \mdspace~as a dual of $\mathcal{W}$, which is a Euclidean space with the same (2D or 3D) dimensionality as  $\mathcal{W}$. In \mdspace, the robot is shrunk to a point at the origin, and the obstacles are transformed into the CO, which explicitly encodes the robot's geometry and configuration, and is defined as follows.

\begin{defi} [\it Configuration Obstacle \cite{Lozano.TC83}] \label{def: CO}
\normalfont
    Given a robot $\mathcal{R}$ and an obstacle $\mathcal{O}$, their  
    CO is defined via the M-sum
    by:
    \begin{align}
        \mathcal{O}^c= \mathcal{O}\oplus (-\mathcal{R})\label{eq: co},
    \end{align}
    where $-\mathcal{R}=\{-\bm{p}\mid \bm{p} \in \mathcal{R}\}$ \textcolor{black}{is the reflection of $\mathcal{R}$ through the origin, and $\bm{p}$ denotes a position vector of $\mathcal{R}$. Note that we refer to the M-sum of $\mathcal{O}$ and $-\mathcal{R}$ as the \textit{Minkowski difference}, following the collision-detection convention.}
\end{defi}
\begin{lemma}[Collision detection via CO
\cite{Ericson}]
    Given two nonempty sets $\mathcal{R},\mathcal{O}\subset\mathbb{R}^d$, they intersect (are in collision) if and only if their CO contains the origin $\bm{0}$.\label{lem: collision}
\end{lemma}

We now briefly review the fundamental properties of the M-sum: (i) the M-sum is commutative, and (ii) the M-sum of a point and a line is a line translated by that point, while the M-sum of two points
is a point. These properties lead to the following result in 2D space.
\begin{lemma}[Properties of the CO in $\mathbb{R}^2$ \cite{LaValle.Planning.06}]
    Given convex, nonempty polygons $\mathcal{R}$ and $\mathcal{O}$ with $\ell_r$ and $\ell_o$ edges, their corresponding CO satisfies the following properties: (i) it is a convex, nonempty polygon with at most $\ell_r+\ell_o$ edges, and can be computed in $O(\ell_r+\ell_o)$ time, and (ii) each of its edges is exactly a translated edge from either $\mathcal{R}$ or $\mathcal{O}$.
    \label{lem: properties}
\end{lemma}

\begin{remark}
    While the collision equivalence in Lemma~\ref{lem: collision} holds for any arbitrary nonempty sets when COs are computed by \eqref{eq: co}, here we assume convex polygons for both the robot and obstacles. Under this assumption, their CO is also a convex polygon by Lemma~\ref{lem: properties}, which is required for the convex-program formulations developed later. Well-behaved non-convex sets (e.g., non-fractal) can still be handled via convex decomposition and pairwise evaluation of Lemma~\ref{lem: collision}. 
\end{remark}

\begin{defi} [\it Signed Distance via CO] 
\normalfont\label{def: sd}
    The signed distance between a robot $\mathcal{R}$ and an obstacle $\mathcal{O}$ is defined as
    \begin{align}\begin{split}
        \mathrm{sd}(\mathcal{R},\mathcal{O}) &=\mathrm{dist}(\mathcal{R},\mathcal{O})-\mathrm{pd}(\mathcal{R},\mathcal{O})\label{eq: sd},
        \end{split}
    \end{align}
    where $\mathrm{dist}(\cdot,\cdot)$ and $\mathrm{pd}(\cdot,\cdot)$ are the minimum distance and penetration depth, and are defined as \cite{Zhang.OBCA.TCST20}
    \begin{subequations}
    \begin{align}
        \mathrm{dist}(\mathcal{R},\mathcal{O})
        &=\inf\{\|\bm{t}\|_2\mid (\mathcal{R}+\bm{t})\cap\mathcal{O}\neq\varnothing\}\label{eq: dist_w_def}
        \\
        &=\inf\{\|\bm{t}\|_2\mid (\bm{0}+\bm{t})\cap\cobs\neq\varnothing\},\label{eq: dist_md_def}\\[0.5em]
        \mathrm{pd}(\mathcal{R},\mathcal{O})
         &=\inf\{\|\bm{t}\|_2\mid (\mathcal{R}+\bm{t})\cap\mathcal{O}=\varnothing\}\label{eq: pd_w_def}\\
        &=\inf\{\|\bm{t}\|_2\mid (\bm{0}+\bm{t})\cap\cobs=\varnothing\}.
        \label{eq: pd_md_def}
    \end{align}\end{subequations}
\end{defi}
\noindent 
The signed distance is positive when $\mathcal{R}$ and $\mathcal{O}$ are disjoint, negative when they intersect, and zero when they are just in contact.
The signed distance between $\mathcal{R}$ and $\mathcal{O}$ can equivalently be expressed as the signed distance between the origin $\bm{0}$ and their CO. Thus, collision avoidance can be enforced by ensuring $\mathrm{sd}(\bm{0},\cobs)>0$. 
The following notion will be useful for the differentiability analysis presented later in Sec.~\ref{subsec: diff-analysis}.
\begin{defi}[Skeleton of a set \cite{Patrikalakis.Maekawa.book}]\label{def: skeleton}
    The skeleton (or medial axis) of the set $\Omega\subset \mathbb{R}^d$ is the closure of the locus of the centers of all maximal inscribed balls within $\Omega$ that touch the boundary at two or more points.
\end{defi}

\begin{lemma}[Properties of SDF in Euclidean space {\cite{Osher.etal.book04}}]
    The SDF $\Phi(\mathbf{x})$ associated with $\Omega\subseteq\mathbb{R}^d$
    is differentiable a.e. in $\mathbb{R}^d$. Specifically, for every $\mathbf{x}$ that does not lie on the skeleton of $\Omega$, its gradient satisfies the eikonal equation $\|\nabla \Phi
    (\mathbf{x}) \|_2=1$ and is equal to the local (outward)  unit normal $N(\mathbf{x})$.
    \label{lem: sdf-prop}
\end{lemma}

\subsection{Nonsmooth Control Barrier Functions}\label{subsec: prelim_ncbf}
As will be shown in Sec.~\ref{subsec: diff-analysis}, our function of interest,
the SDF between polytopic sets, is locally Lipschitz but
not $C^1$. 
Therefore, the classical CBF formulation \cite{Ames.etal.ECC19} cannot be directly applied. We instead use NCBFs, proposed in \cite{Glotfelter.etal.LCSS17} as a generalization of CBFs, to establish safety guarantees. The required basic notions of nonsmooth analysis from \cite{Clarke.1975} are presented below.
Consider nonlinear control-affine systems:
\begin{align}
    \dot{\bm{x}}=f(\bm{x})+g(\bm{x})\bm{u},\label{eq: ctrl-affine}
\end{align}
where $\mathcal{X}\subset\mathbb{R}^n$ is an open connected set, $f:\mathcal{X}\rightarrow \mathbb{R}^n$, $g:\mathcal{X}\rightarrow \mathbb{R}^{n\times q}$ are locally Lipschitz, $\bm{x}\in \mathcal{X}$, and $\bm{u}\in \mathcal{U}\subset\mathbb{R}^q$ ($\mathcal{U}$ denotes a closed control constraint set). 

Similar to classical $C^1$ CBFs \cite{Ames.etal.ECC19}, safety is guaranteed by the forward invariance of a safe set $C = \{\bm{x}\in \mathcal{X}\mid h(\bm{x}) \geq 0\}$, defined by a locally Lipschitz function
$h:\mathcal{X}\to\mathbb{R}$.
This requires evaluating the rate of change of $h$ along the trajectories of system dynamics. To handle the “bad” points at which no gradient exists, we use the generalized gradient defined below.
\begin{defi} [\it Generalized Gradient \cite{Cortes.08.CSM}]  
\normalfont
    Let $h:\mathcal{X}\to\mathbb{R}$ be a locally Lipschitz function, and $S_h\subset\mathcal{X}$ be the zero-measure set where $h$ is non-differentiable, and suppose $S$ is any Lebesgue zero-measure set in $\mathbb{R}^n$. Then, the generalized gradient 
    of $h$ at $\bm{x}$ is:
    \begin{align}
        \partial h(\bm{x}) = \mathrm{co}\left\{\lim_{i\to\infty}\nabla h(\bm{x}_i): \bm{x}_i\to \bm{x},~\bm{x}_i\notin S\cup S_h \right\}, \label{eq: gen_grad_h}
    \end{align}
\end{defi}
\noindent where $\mathrm{co}$ denotes the convex hull, $\{\bm{x}_i\}_{i=1}^\infty$ denotes a sequence of points in $\mathcal{X}$ converging to $\bm{x}$.
Consider the measurable and locally bounded feedback controller $\bm{u}:\mathcal{X}\to\mathcal{U}$, which is not necessarily continuous. 
To provide an appropriate solution notion for the resulting (possibly discontinuous) closed-loop dynamics, we use the Filippov set-valued map $K[f+g\bm{u}]:\mathcal{X}\to2^{\mathbb{R}^n}$ to turn \eqref{eq: ctrl-affine} into a differential inclusion:
\begin{align}
    &\dot{\bm{x}}(t)\in K[f+g\bm{u}](\bm{x}(t)) \label{eq: diff-inclusion}\\
    &=
    \mathrm{co}\left\{\lim_{i\to\infty}f(\bm{x}'_i)+g(\bm{x}'_i)\bm{u}(\bm{x}'_i): \bm{x}'_i\to \bm{x}(t),~\bm{x}'_i\notin S\cup S_f \right\},\nonumber 
\end{align}
where 
$S, S_f\subset\mathcal{X}$ are arbitrary and system-dependent zero-measure sets, respectively, and the map $K[f+g\bm{u}]$
is assumed to be upper semicontinuous with nonempty, compact, and convex values,
ensuring the existence of a Carathéodory solution\footnote{A Carathéodory solution is an absolutely continuous trajectory that satisfies system dynamics for almost all time, in the sense of Lebesgue measure.} to \eqref{eq: diff-inclusion}. 
We refer readers to \cite{Cortes.08.CSM} for an extensive study of set-valued maps and discontinuous dynamical systems.

Having established the generalized gradient and the Filippov set-valued map,
the set-valued generalization of the inner product is defined as $\left\langle\partial h(\bm{x}), K[f+g\bm{u}](\bm{x}) \right\rangle:=\{a\in\mathbb{R}:\exists \bm{\zeta}\in\partial h(\bm{x}), \exists \bm{w}\in K[f+g\bm{u}](\bm{x}), \left\langle\bm{\zeta},\bm{w}\right\rangle=a\}$.
Combining the forward-invariance result in \cite[Thm.~3]{Glotfelter.etal.LCSS17} with the valid-NCBF condition in \cite[Def.~4]{Glotfelter.etal.CCTA18}, the following theorem gives the safety condition used in this paper.
\begin{thm}[Safety Guarantee {\cite{Glotfelter.etal.LCSS17,Glotfelter.etal.CCTA18}}]
    \label{thm: safety-guarantee}
    Given a nonempty safe set $C$ defined by a locally Lipschitz function $h:\mathcal{X}\to\mathbb{R}$, if there exists a locally Lipschitz extended class-$\mathcal{K}$ function\footnote{%
    An extended class-$\mathcal{K}$ function is a function $\alpha:\mathbb{R}\rightarrow\mathbb{R}$ that is strictly increasing and $\alpha(0)=0$ \cite{khalil.02}.} $\alpha$ and a measurable, locally bounded $\bm{u}:\mathcal{X}\to\mathcal{U}$ such that
    \begin{align}
        \min \left\langle\partial h(\bm{x}), K[f+g\bm{u}](\bm{x}) \right\rangle \geq -\alpha(h(\bm{x})),\,\,\forall\bm{x}\in\mathcal{X}, \label{eq: ncbf-set-valued-constr}
    \end{align}
    then $h$ is a valid NCBF and the control law $\bm{u}$ renders the set $C$ forward invariant (i.e., all Carathéodory solutions that start in the set stay in the set) for \eqref{eq: diff-inclusion}.
\end{thm}
\begin{remark}
    At any state $\bm{x} \in \mathcal{X}$ where $h$ is continuously differentiable and the given control law $\bm{u}$ is locally Lipschitz in a neighborhood of $\bm{x}$, both $\partial h(\bm{x})$ and $K[f+g\bm{u}](\bm{x})$ reduce to singletons $\{\nabla h(\bm{x})\}$ and $\{f(\bm{x})+g(\bm{x})\bm{u}(\bm{x})\}$, respectively. 
    Hence, the NCBF constraint in \eqref{eq: ncbf-set-valued-constr} reduces to the classical CBF constraint $\nabla h(\bm{x})^\top(f(\bm{x})+g(\bm{x})\bm{u}(\bm{x}))\geq-\alpha(h(\bm{x}))$ \cite{Ames.etal.ECC19}.
\end{remark}

\subsection{Differentiable Optimization} \label{subsec: prelim_diffopt}
Since our SDF and hence the resulting NCBF are defined implicitly via convex programs parameterized by the system state, evaluating their exact gradients requires characterizing how the optimal solution and value function vary locally with the state.
This mathematical foundation is known as sensitivity analysis in parametric programming \cite{Fiacco.90}, and, more recently, as differentiable optimization (diffOpt) in machine learning \cite{Amos.Kolter.ICML17}.
We use sensitivity analysis and diffOpt interchangeably. Consider a parametric inequality-constrained QP:
\begin{align}\begin{split}
    &\mathbf{z}^*=\arg\min_{\mathbf{z}} f_0(\mathbf{z}),\quad\mathrm{s.t.}\,\,\, \mathbf{G}(\mathbf{x})\mathbf{z}\leq \mathbf{h}(\mathbf{x}),\\
    &\text{with} \,\,\,\,f_0(\mathbf{z})=\mathbf{z}^\top \mathbf{Q}\mathbf{z}+\mathbf{p}^\top\mathbf{z},\quad
    \mathbf{Q}=\mathbf{Q}^\top\succ0,
    \,\,\mathbf{p}\in\mathbb{R}^k
    \label{eq: parametric-qp}
\end{split}\end{align}
where $\mathbf{z}\in\mathbb{R}^k$ is the decision variable, $\mathbf{x}\in\mathbb{R}^p$ is the parameter, 
and $\mathbf{G}(\mathbf{x})\in\mathbb{R}^{\mathfrak{g}\times k}$, $\mathbf{h}(\mathbf{x})\in\mathbb{R}^\mathfrak{g}$ define the $\mathfrak{g}$ parameter-dependent inequality constraints. We assume $\mathbf{G}(\mathbf{x})$ and $\mathbf{h}(\mathbf{x})$ are continuously differentiable with respect to $\mathbf{x}$. Let the active constraint set be $\mathcal{I}(\mathbf{x}) := \{1\leq i\leq \mathfrak{g} \mid \mathbf{G}_i(\mathbf{x}) \mathbf{z}^*(\mathbf{x}) = \mathbf{h}_i(\mathbf{x}) \}$, where the subscript $i$ denotes the $i$-th row of the matrices. Let $\boldsymbol{\lambda}(\mathbf{x})\in\mathbb{R}^\mathfrak{g}_{\geq0}$ be the associated dual variable. For brevity, we omit the dependence on $\mathbf{x}$ when clear from context. 
Define the concatenated primal-dual vector $\bm{\xi}:= [\mathbf{z}^\top, \boldsymbol{\lambda}^\top]^\top \in \mathbb{R}^{n_{\bm{\xi}}}$,
with $n_{\bm{\xi}} =k + \mathfrak{g}$. 
The Lagrangian of \eqref{eq: parametric-qp} is
\begin{align}
    \mathcal{L}(\bm{\xi};\mathbf{x}):= f_0(\mathbf{z})+ \boldsymbol{\lambda}^\top (\mathbf{G}(\mathbf{x})\mathbf{z}-\mathbf{h}(\mathbf{x})). \label{eq: L}
\end{align}
The KKT residual $\Gamma(\bm{\xi};\mathbf{x})\in\mathbb{R}^{n_{\bm{\xi}}}$ is defined as
\begin{align}
    \Gamma(\bm{\xi};\mathbf{x}):=\begin{bmatrix}
    2\mathbf{Q}\mathbf{z}+\mathbf{p} + \mathbf{G}(\mathbf{x})^\top\boldsymbol{\lambda}
 \\D(\boldsymbol{\lambda})(\mathbf{G}(\mathbf{x})\mathbf{z}-\mathbf{h}(\mathbf{x}))
\end{bmatrix},
\label{eq: KKT-residual}
\end{align}
where the first and second block rows correspond to the stationarity and complementary slackness conditions, respectively. Here, $D(\cdot)$ creates a diagonal matrix from a vector. 

The local sensitivity result presented later relies on the following regularity conditions:
\begin{enumerate}
    \item Linear Independence Constraint Qualification (LICQ): This is satisfied if the gradients of all active constraints  with respect to $\mathbf{z}$,  $\{\mathbf{G}_i^\top \}_{i\in\mathcal{I}}$, are linearly independent. 
    \item Strict Complementarity (SC): This condition holds if the dual variables associated with all active constraints are strictly positive, i.e., $\boldsymbol{\lambda}_i>0,\,~\forall i\in \mathcal{I}$.
    \item Second-Order Condition (SOC): This condition holds if the Hessian of the Lagrangian with respect to $\mathbf{z}$ is positive definite on the tangent space of the active constraints. 
    Since $\nabla_{\mathbf{z}}^2 \mathcal{L}(\bm{\xi};\mathbf{x})=2\mathbf{Q}\succ0$, the SOC holds automatically for \eqref{eq: parametric-qp}. 
\end{enumerate}
Under these conditions, the following theorem, adapted from \cite[Theorem~4.4]{Still.note.18}, applies the Implicit Function Theorem (IFT) to characterize the local sensitivities of the optimal solution and value function. We remark that the result only holds locally when the active constraint set remains unchanged.
\begin{thm} [Local stability based on IFT {\cite{Still.note.18}}] \label{thm: diffopt-ift}
    \normalfont
    Let $(\mathbf{z}^*, \boldsymbol{\lambda}^*)$ be an optimal primal-dual pair satisfying the KKT conditions of \eqref{eq: parametric-qp} and let $\bm{\xi}^*:= [(\mathbf{z}^*)^\top, (\bm{\lambda}^*)^\top]^\top$,
    so that 
    $\Gamma(\bm{\xi}^*;\mathbf{x})=0$.
    Assuming LICQ, SC, and SOC hold,
    there exists a neighborhood $\mathcal{N}_{\mathbf{x}}$ of $\mathbf{x}$ in which the optimal primal-dual pair $\bm{\xi}^*(\mathbf{x})$ is unique and continuously differentiable. Applying the IFT to the KKT system yields the following results:
    \begin{enumerate}
        \item The gradient of the optimal value function $f_0(\mathbf{z}^*(\mathbf{x}))$ with respect to $\mathbf{x}$ is equal to the partial derivative of the Lagrangian at the optimum:
            \begin{align}
                \nabla_{\mathbf{x}}f_0(\mathbf{z}^*(\mathbf{x}))= \frac{\partial \mathcal{L} }{\partial \mathbf{x}}\bigg\vert_
                {\bm{\xi}=\bm{\xi}^*(\mathbf{x})}.
                \label{eq: dLdx}
            \end{align}
        \item The Jacobian of the primal-dual pair $\bm{\xi}^*(\mathbf{x})$ with respect to $\mathbf{x}$ is uniquely determined by
        \begin{align}
            \frac{\partial \bm{\xi}^*}{\partial \mathbf{x}}=-\left[\frac{\partial \Gamma(\bm{\xi}^*)}{\partial \bm{\xi}}\right]^{-1}\frac{\partial \Gamma(\bm{\xi}^*)}{\partial \mathbf{x}}.
            \label{eq: dXidx}
        \end{align}
        
    \end{enumerate}
\end{thm}
\begin{remark}
    For the strictly convex QP in \eqref{eq: parametric-qp}, SOC holds automatically; therefore, under LICQ and SC, the KKT Jacobian $\partial_{\bm{\xi}}\Gamma$ is nonsingular. The value-function sensitivity in \eqref{eq: dLdx}, however, also applies to parametric LPs, i.e., when $\mathbf{Q}=\mathbf{0}$, without requiring SOC \cite[Theorem~4.3]{Still.note.18}.
    \label{remark: LP-dL}
\end{remark}

Under SC, the full KKT sensitivity system in \eqref{eq: dXidx} locally reduces to its active components.
Let $\bm{\xi}_\mathcal{I}^*:= [(\mathbf{z}^*)^\top, (\bm{\lambda}_\mathcal{I}^*)^\top]^\top$.
Then, $\frac{\partial \bm{\xi}_\mathcal{I}^*}{\partial \mathbf{x}}=-[(\partial_{\bm{\xi}}\Gamma)_\mathcal{I}]^{-1}(\partial_{\mathbf{x}}\Gamma)_\mathcal{I}$, with the equivalent reduced Jacobian matrices given by:
\begin{subequations}
    \begin{align}
        (\partial_{\bm{\xi}}\Gamma)_\mathcal{I}
        &:= 
        \begin{bmatrix}
        2\mathbf{Q} & \mathbf{G}
        _\mathcal{I}(\mathbf{x})^\top \\\mathbf{G}_\mathcal{I}(\mathbf{x}) & \bm{0}\\
        \end{bmatrix}\in\mathbb{R}^{n_{\mathcal{I}}\times n_{\mathcal{I}}}, \label{eq: dGdXi}
    \\[0.2em]
    (\partial_{\mathbf{x}}\Gamma)_\mathcal{I}&:=
    \begin{bmatrix}
    (\partial_{\mathbf{x}}\mathbf{G}_\mathcal{I}(\mathbf{x})^\top)\boldsymbol{\lambda}_\mathcal{I}^* \\ (\partial_{\mathbf{x}}\mathbf{G}_\mathcal{I}(\mathbf{x}))\mathbf{z}^*-\partial_{\mathbf{x}}\mathbf{h}_\mathcal{I}(\mathbf{x})
    \end{bmatrix}\in\mathbb{R}^{n_{\mathcal{I}}\times p}.\label{eq: dGdx}
\end{align}\end{subequations}
Here, $n_{\mathcal{I}}=k +|\mathcal{I}|$ is the dimension of the reduced KKT system, with $|\mathcal{I}|$ being the number of active constraints. Thus, the primal and active dual sensitivities, $\frac{\partial\mathbf{z}^*}{\partial{\mathbf{x}}}$ and $\frac{\partial\lambdai}{\partial{\mathbf{x}}}$, are given by 
the first $k$ rows and the last $|\mathcal{I}|$ rows of $\frac{\partial \bm{\xi}_\mathcal{I}^*}{\partial \mathbf{x}}$, respectively.

\section{Problem Formulation}\label{sec: 3}
We now formally state the problem we are solving, after first defining additional notation and assumptions.
Let $\bm{x}\in\mathcal{X}$ be the robot's state. 
We use halfspace representations (H-reps) to define polytopic sets of the robot, obstacles, and their corresponding COs, denoted by $\mathcal{R}(\bm{x})\subset\mathcal{W}\subset\mathbb{R}^d$, $\mathcal{O}_j\subset\mathcal{W}\subset\mathbb{R}^d$, and $\mathcal{O}_j^c(\bm{x})\subset\mathcal{M}_\mathcal{D}\subset\mathbb{R}^d$, respectively.
\begin{subequations}
    \begin{align}
        &\mathcal{R}(\bm{x})= \{\mathbf{y}\in \mathbb{R}^{d}\mid A_r(\bm{x}) \mathbf{y}\leq b_r(\bm{x})\}, \\
        &\mathcal{O}_j= \{\mathbf{y}\in \mathbb{R}^{d}\mid A_{o_j} \mathbf{y}\leq b_{o_j}\},\quad j=1,\ldots,N_\mathcal{O},
        \label{eq: obs-w}\\
        &\mathcal{O}_j^c(\bm{x})=\{\mathbf{y}\in \mathbb{R}^{d}\mid \mathbf{A}^{c_j}(\bm{x}) \mathbf{y}\leq \mathbf{b}^{c_j}(\bm{x})\},\label{eq: obs-md}
    \end{align}
\end{subequations}
where $A_r\in\mathbb{R}^{\ell_r\times d},~b_r\in\mathbb{R}^{\ell_r}$,  $A_{o_j}\in\mathbb{R}^{\ell_{o_j}\times d},~b_{o_j}\in\mathbb{R}^{\ell_{o_j}}$, $\mathbf{A}^{c_j}\in\mathbb{R}^{\ell
_{c_j}\times d},~\mathbf{b}^{c_j}\in\mathbb{R}^{\ell
_{c_j}}$, with $\ell_r$, $\ell_{o_j}$, and $\ell_{c_j}$ denoting the number of facets of the robot, the $j$-th of the $N_\mathcal{O}$ obstacles, and their CO, respectively; $d\in\{2,3\}$ for 2D and 3D spaces. 
Each CO depends on $\bm{x}$ via \eqref{eq: co} and is used to evaluate the SDF for each robot-obstacle pair. We denote the CO by $\cobs$ and $\mathcal{CO}$ interchangeably, and omit the state dependency of $\mathcal{R},\cobs_j$ when clear from context.

Throughout this paper, the superscript $c$ denotes quantities associated with the CO in \mdspace; the subscripts $i$ and $\mathcal{I}$ are used to extract the $i$-th row and the active component(s) from a given matrix, respectively.
We use $(\mathbf{A}^c_{\mathcal{I}}, \mathbf{b}^c_{\mathcal{I}})$ to denote the active constraint(s) of the CO, and $(\aci, \bcci)$ when emphasizing a single active constraint. 
Under the LICQ assumption, the cardinality of the active set is bounded by the spatial dimension, i.e., $1 \leq |\mathcal{I}| \leq d$. 
For a nonempty closed set $\Omega \subset \mathbb{R}^d$, we denote its boundary, interior, and skeleton as $\mathrm{bd}(\Omega)$, $\mathrm{int}(\Omega)$ and $\mathcal{S}(\Omega)$, respectively. For a simple polygon $\mathcal{P} \subset \mathbb{R}^2$, we denote its vertex and (open) edge sets by $\mathcal{V}(\mathcal{P})$ and $\mathcal{E}(\mathcal{P})$, respectively, and its boundary consists
of its vertices and the edges (open line segments), i.e., $\mathrm{bd}(\mathcal{P})=\mathcal{V}(\mathcal{P})\cup\mathcal{E}(\mathcal{P})$. We note that according to Definition~\ref{def: skeleton}, the vertices of $\mathcal{P}$ are part of its skeleton, i.e., $\mathcal{V}(\mathcal{P})\subset\mathcal{S}(\mathcal{P})$.

Our problem is defined based on the following assumptions:
\begin{enumerate}
    \item The robot is modeled as a compact, convex polytopic rigid body in $\mathbb{R}^d$ with nonempty interior.
    For each $\bm{x}\in \mathcal{X}$,
    the H-rep functions ${A}_r(\bm{x})$ and $b_r(\bm{x})$ defining its 
    occupied set $\mathcal{R}(\bm{x})$ 
    are $C^1$
    with respect to $\bm{x}$. \label{assp1}
    \item Each obstacle $\mathcal{O}_j$ is a compact, convex, static polytope with nonempty interior. Information regarding the obstacles is accessible to the robot. \label{assp2}
    \item For almost all $\bm{x}\in \mathcal{X}$, the induced CO H-rep functions $\mathbf{A}^{c_j}(\bm{x})$ and $\mathbf{b}^{c_j}(\bm{x})$ are continuously differentiable with respect to $\bm{x}$, and the associated convex programs satisfy LICQ and SC.
    In 2D, geometric degeneracies occur (on a zero-measure set) when $\mathcal{R}$ and $\mathcal{O}_j$ have parallel edges, causing the minimal CO H-rep to contain fewer than $\ell_r+\ell_{o_j}$ rows.
    Their effect on the differentiability of the SDF will be discussed further in Sec.~\ref{subsec: diff-analysis}. \label{assp3}
    \item All H-reps of polytopic sets are minimal.
    \label{assp4}
\end{enumerate}
We now formulate the safe-navigation problem for a polytopic robot $\mathcal{R}$ among polytopic obstacles $\{\mathcal{O}_j\}_{j=1}^{N_\mathcal{O}}$.
\begin{prob}
    \label{prob: 1}
    Given a robot $\mathcal{R}$ with known control-affine dynamics in the form of \eqref{eq: ctrl-affine} and a collection of static obstacles $\{\mathcal{O}_j\}_{j=1}^{N_\mathcal{O}}$, find a state-feedback controller $\bm{u}^*(\bm{x})$
    that
    \begin{enumerate}
        \item 
        %
        minimizes a quadratic cost pointwise and steers the robot to a goal position $\bm{x}_g\in(\mathcal{W}\setminus \bigcup_j\mathcal{O}_j)$;
        \item 
        satisfies (i) $\text{sd}(\mathcal{R}(\bm{x}),\mathcal{O}_j) \geq d_{\mathrm{safe}},~\forall j$, where $d_{\mathrm{safe}}\geq 0$ is the user-defined safety margin, and (ii) the control bounds $\bm{u}^*(\bm{x})\in\mathcal{U}:=\{\bm{u}\in\mathbb R^q\mid\bm{u}_{\min}\leq\bm{u}\leq\bm{u}_{\max}\}$.
    \end{enumerate} 
    \label{prob}
\end{prob}
Subsequent derivations suppress the index $j$; the resulting construction applies pairwise to each $\mathcal{O}_j$.

\section{Convex-optimization-based Signed Distance
}\label{sec: 4}
This section presents our approach to computing the ``exact'' SDF between polytopes using convex optimization in \mdspace. Specifically, we compute the minimum distance via a QP (\ref{subsec: dist-qp}) and the penetration depth via a novel LP (\ref{subsec: depth-lp}). We then investigate the equivalence of their optimality conditions from both geometric (\ref{subsec: geo-equi}) and algebraic (\ref{subsec: algebra-equi}) perspectives. Finally, we study the differentiability of the exact SDF in SE(2) (\ref{subsec: diff-analysis}).

\subsection{Minimum Distance as a Quadratic Program}\label{subsec: dist-qp}
When the robot $\mathcal{R}(\bm{x})$ and an obstacle $\mathcal{O}$ are disjoint in \wspace, i.e., $\mathcal{R}(\bm{x})\cap \mathcal{O}=\varnothing$, the minimum distance between the two sets can be computed by solving the following QP:
\begin{align}\begin{split}
    &\min\,\,\|\mathbf{r}-\mathbf{o}\|^2_2
   \quad\mathrm{s.t.}\,\,\, A_r(\bm{x}) \mathbf{r}\leq b_r(\bm{x}), \,\,A_o\mathbf{o} \leq b_o\\
   &\,\,\text{with}\,\,\,\, \mathrm{dist}(\mathcal{R}(\bm{x}),\mathcal{O}) =\|\mathbf{r}^*(\bm{x})-\mathbf{o}^*(\bm{x})\|_2,
   \label{eq: dist_w}
\end{split}\end{align}
where 
$(\mathbf{r}^*,\mathbf{o}^*)
=(\mathbf{r}^*(\bm{x}),\mathbf{o}^*(\bm{x}))
$ is the optimal solution to \eqref{eq: dist_w}, 
representing the closest points on the robot-obstacle pair. It is worth noting that the closest-point pair $(\mathbf{r}^*,\mathbf{o}^*)$ in \wspace~is not necessarily unique, as in 2D cases in which the closest features are parallel edges
(Fig.~\ref{fig: cvxopt-4cases}). In contrast, the equivalent formulation in \mdspace~given below
yields a unique optimal solution, as we will establish in Lemma~\ref{lem: zstar-unique-qp}.
\begin{align}\begin{split}
    &\min\,\,\|\bm{z}\|^2_2\quad\mathrm{s.t.}\,\,\,
    \mathbf{A}^c(\bm{x}) \bm{z}\leq \mathbf{b}^c(\bm{x})
    \\
   &\,\,\text{with}\,\,\,\, 
   \mathrm{dist}(\mathcal{R}(\bm{x}),\mathcal{O})=\mathrm{dist}(\bm{0},\cobs(\bm{x})) =\|\bm{z}^*(\bm{x})\|_2,
   \label{eq: dist_md}
\end{split}\end{align}
where $\bm{z}^*=\bm{z}^*(\bm{x})\in\mathbb{R}^d$ is the optimal solution of \eqref{eq: dist_md}, whose 2-norm gives the minimum distance. Furthermore, 
the optimal solution $\bm{z}^*$ is in fact the (Euclidean) projection of the origin $\bm{0}$ onto the CO \cite[Chapter~8.1.1]{Boyd.Vandenberghe.04}.

\begin{defi}[Projection of a point onto a set {\cite{Boyd.Vandenberghe.04}}]\label{def: proj}
    Let $C$ be a nonempty closed subset of $\mathbb{R}^n$. The projection of a point $\mathbf{x}\in\mathbb{R}^n$ onto $C$, denoted as $\mathrm{proj}_C(\mathbf{x})$, is defined as
    \begin{align}
        \mathrm{proj}_C(\mathbf{x})=\arg\min\{\|\bm{z}-\mathbf{x}\|_2\mid\bm{z}\in C\}. \label{eq: proj}
    \end{align}
\end{defi}
\begin{lemma}
    The optimal solution $\bm{z}^*$ of \eqref{eq: dist_md} is the unique projection of the origin $\bm{0}$ onto the CO in \mdspace. Given that $\mathcal{R}\cap \mathcal{O}= \varnothing$, it lies on $\mathrm{bd}(\cobs)$.
    \label{lem: zstar-unique-qp}
\end{lemma}
\begin{proof}
    Since minimizing $\|\bm{z}\|_2^2$ is equivalent to minimizing $\|\bm{z}\|_2$, \eqref{eq: dist_md} coincides with the projection problem in Definition~\ref{def: proj}; hence, $\bm{z}^*$ is a projection of the origin onto the CO. Next, by Lemma~\ref{lem: properties}, the CO is a nonempty closed convex set in $\mathbb{R}^d$, so this projection is unique by 
    the Projection Theorem \cite[Proposition~1.1.9]{Bertsekas.book.09}. Moreover, $\mathcal{R}\cap \mathcal{O}= \varnothing$ implies
    $\bm{0}\notin\cobs$ by Lemma~\ref{lem: collision}. If $\bm{z}^*\in\mathrm{int}(\cobs)$, 
    an open neighborhood of $\bm{z}^*$  contained in $\cobs$ would include a point strictly closer to the origin, contradicting optimality. Therefore, $\bm{z}^* \in \mathrm{bd}(\cobs)$. 
\end{proof}

\begin{remark}[Triviality of the minimum-distance QP under collision]
    When the sets are in collision, the origin lies inside the CO, i.e., $\mathbf{0} \in \mathrm{int}(\cobs)$, so \eqref{eq: dist_md} trivially returns $\bm{z}^*=\bm{0}$. Thus, it provides no measure of penetration, motivating the penetration-depth formulation introduced next.
\end{remark}

\subsection{Penetration Depth as a Linear Program}\label{subsec: depth-lp}
We know from Definition~\ref{def: sd} that the penetration depth is the infimum of the translation magnitudes required to separate two intersecting sets \eqref{eq: pd_w_def};
however, it is unclear how to impose the separation condition, $(\mathcal{R}+\bm{t})\cap \mathcal{O} = \varnothing$, as a convex constraint in \wspace. 
In contrast to the minimum-distance problem, which admits QP formulations in both \wspace~\eqref{eq: dist_w} and \mdspace~\eqref{eq: dist_md}, the penetration-depth problem lacks a known convex optimization formulation, to the best of the authors' knowledge. To overcome this limitation, our preliminary work \cite{Chen.etal.CDC25} proposed a novel LP formulation in \mdspace~by recasting the penetration-depth computation as the standard point-depth problem \cite[Chapter~8.5]{Boyd.Vandenberghe.04}, i.e., the depth of the origin with respect to the CO, given by
\begin{align}\begin{split}
    & \mathrm{pd}(\mathcal{R}(\bm{x}),\mathcal{O})=\mathrm{depth}(\bm{0},\cobs(\bm{x})) =\,\,{\max}_{s\in\mathbb{R}^+}\,\, s
     \\
   &\,\,\mathrm{s.t.}
   \,\,\,\,
   \|\mathbf{a}^c_{i}(\bm{x})\|_2 s \leq {b}^c_{i}(\bm{x}),\,\, i=1,\ldots,\ell_c,
   \label{eq: pd_md}
\end{split}\end{align}
where $\mathbf{a}^c_{i}$ and ${b}^c_{i}$ denote the $i$-th row of $\mathbf{A}^c$ and $\mathbf{b}^c$, respectively. 
The optimal value admits the closed form 
$s^*=\min_i{b}^c_{i}(\bm{x})/\|\mathbf{a}^c_{i}(\bm{x})\|_2$,
i.e., the distance from the origin to the closest hyperplane of the CO. For the rest of this paper, we refer to the convex programs in \mdspace~\eqref{eq: dist_md} and \eqref{eq: pd_md} as the dist-QP and depth-LP, respectively. For later use, we define the active constraint sets of the two programs as $\mathcal{I}^\mathrm{QP}(\bm{x})=\{i\mid\mathbf{a}^c_{i}(\bm{x})\bm{z}^*={b}^c_{i}(\bm{x})\}$ and $\mathcal{I}^\mathrm{LP}(\bm{x})=\{i\mid \|\mathbf{a}^c_{i}(\bm{x})\|_2s^*={b}^c_{i}(\bm{x})\}$. When the context is clear, we simply write $\mathcal{I}$.

While \eqref{eq: pd_md} provides the penetration depth $s^*$, it does not provide any directional information required for the subsequent SDF-gradient computation.
To obtain a representation analogous to the dist-QP solution in \eqref{eq: dist_md}, we recover $\bm{z}_i^* \in \mathrm{bd}(\mathcal{O}^c)$,
referred to as witness point(s)\footnote{While our prior work \cite{Chen.etal.CDC25, Chen.etal.ArXiv26} denoted $\bm{z}^*$ as a ``critical point'', we adopt the term ``witness point'' from the collision detection community in this article to prevent any confusion with the standard definition of critical points in optimization and nonsmooth analysis.}, by projecting the origin onto the hyperplane associated with the active constraint of \eqref{eq: pd_md}, i.e., $\mathcal{H}_i(\bm{x})=\{\mathbf{y}\in \mathbb{R}^{d}\mid \mathbf{a}^c_{i}(\bm{x}) \mathbf{y} = {b}^c_{i}(\bm{x}),\forall i\in\mathcal{I}^\mathrm{LP}\}$:
\begin{align}\begin{split}
    \bm{z}^*_i(\bm{x}) = \mathrm{proj}_{\mathcal{H}_i}(\bm{0}) = \frac{{b}^c_{i}(\bm{x})}{\| \mathbf{a}^c_{i}(\bm{x})\|_2^2} \left(\mathbf{a}^c_{i}(\bm{x})\right)^\top,\,\,\forall i\in\mathcal{I}^\mathrm{LP},
\label{eq: pd_zstar}
\end{split}\end{align}
with $\|\bm{z}_i^*\|_2 =s^*$ by the definition of $\mathcal{I}^\mathrm{LP}$. If $|\mathcal{I}^\mathrm{LP}|=1$, the witness point is unique and is denoted by $\bm{z}^*$. If $|\mathcal{I}^\mathrm{LP}|>1$, the penetration-depth value $s^*$ remains unique, but the witness point is not. Such non-uniqueness occurs when the origin lies on the skeleton (Definition~\ref{def: skeleton}) of the CO, corresponding to non-differentiable points of the SDF, which constitute a set of measure zero, denoted by $\mathcal{X}_\mathcal{S} = \{\bm{x} \in \mathcal{X} \mid \bm{0} \in \mathcal{S}(\cobs(\bm{x}))\}$. In the following sections, we investigate the relationship between the optimality conditions of the dist-QP \eqref{eq: dist_md} and the depth-LP \eqref{eq: pd_md} for $\bm{x}\notin\mathcal{X}_\mathcal{S}$. 


\begin{remark}
    The zero-measure set $\mathcal{X}_\mathcal{S}$ must be non-empty unless the CO is a halfspace (which will not be the case under our assumptions on the robot and obstacles), which justifies the inherent nonsmoothness
    of the SDF discussed earlier. 
\end{remark}


\subsection{Geometric Equivalence of the Witness Point} \label{subsec: geo-equi}
The key advantages of formulating the optimization problems in \mdspace~are twofold: (i) it enables a convex optimization (LP) formulation for penetration depth computation \eqref{eq: pd_md}, and (ii) the minimum distance and penetration depth share an identical geometric interpretation, i.e., regardless of the collision status, the witness point $\bm{z}^*$ is geometrically equivalent to the Euclidean projection of the origin onto the boundary of the CO as established by
Lemma~\ref{lem: zstar-unique-qp} and the witness-point construction in \eqref{eq: pd_zstar}, as shown in Fig.~\ref{fig: sd}. The signed distance $\mathrm{sd}(\mathcal{R}, \mathcal{O})$ can thus be defined using the unified metric based on the Euclidean norm of $\bm{z}^*$:
\begin{align}
\mathrm{sd}(\mathcal{R}, \mathcal{O}) = \mathrm{sd}(\bm{0}, \cobs) =\begin{cases}
+\|\bm{z}^*\|_2 & \text{if } \mathbf{0} \notin \cobs
\\
-\|\bm{z}^*\|_2 & \text{if } \mathbf{0} \in \cobs
\end{cases}. \label{eq: sd_as_2norm}
\end{align}

\begin{figure}[hbt!]
    \centering\includegraphics[width=0.95\linewidth]{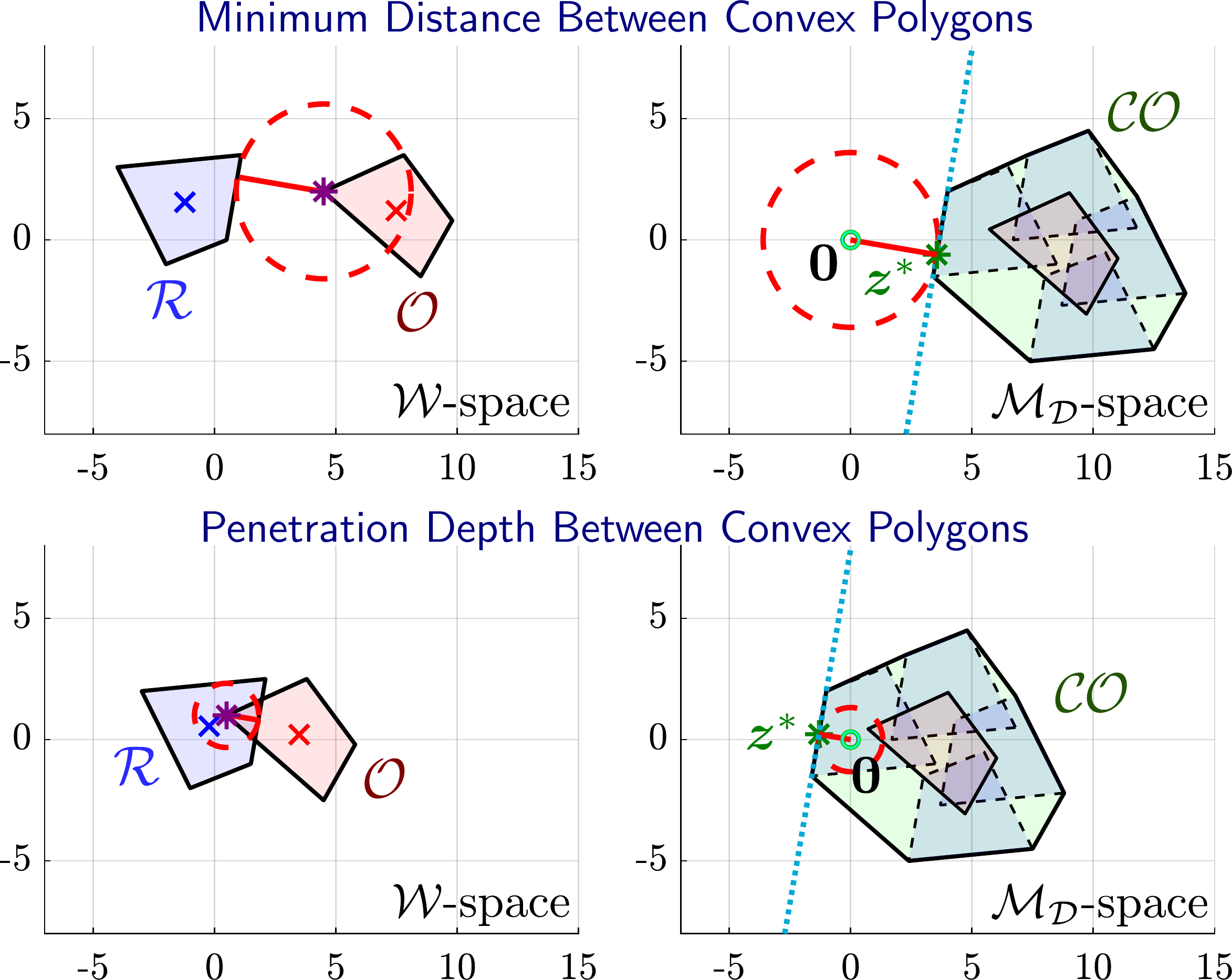}
    \caption{Distances between disjoint (\textbf{Top}) and intersecting (\textbf{Bottom}) convex sets $\mathcal{R}$ and $\mathcal{O}$ in \wspace~(\textbf{Left}), and the corresponding equivalent distances in \mdspace~(\textbf{Right}). In \mdspace, the origin $\bm{0}$ lies outside or inside of the CO if the signed distance is positive or negative, respectively. Dashed lines denote the active constraint(s) associated with the witness point $\bm{z}^*$.}
    \label{fig: sd}
\end{figure}
The exactness of \eqref{eq: sd_as_2norm} follows from Lemma~\ref{lem: zstar-unique-qp} in the safe case and from the depth-LP
formulation \eqref{eq: pd_md} in the penetration case; hence, \eqref{eq: sd_as_2norm} coincides with the standard SDF definition \cite{Osher.etal.book04}. This representation is continuous across $\mathrm{bd}(\cobs)$, where $\|\bm{z}^*\|_2=0$, and enables the algebraic correspondence and unified
gradient computation developed next.

\subsection{Algebraic Correspondence of dist-QP and depth-LP
}
\label{subsec: algebra-equi}
We now establish the algebraic correspondence between the dist-QP \eqref{eq: dist_md} and the depth-LP \eqref{eq: pd_md} under the single-active-constraint condition. The single-active condition corresponds to $\bm{z}^*\in\mathcal{E}(\cobs)$ in the safe case, and holds a.e. in the colliding case, where the LP-recovered witness point $\bm{z}^*$ is uniquely determined by \eqref{eq: pd_zstar}.
Specifically, we show that the LP optimal dual variable $\nu^*_\mathcal{I}$ can be rescaled into an LP-induced pseudo multiplier $\lambda_\mathcal{I}^*$, yielding a signed ``pseudo-QP'' tuple that matches the single-active stationarity structure of the dist-QP.
This mapping is given in the following lemma.
\begin{lemma}
    \label{lem: pseudo-qp-mapping}
    Suppose that $\bm{0}\in\mathrm{int}(\cobs(\bm{x}))$ and $\bm{x}\notin\mathcal{X}_\mathcal{S}$. Then the optimal primal-dual pair $(s^*,\nu^*_\mathcal{I})$ of the depth-LP \eqref{eq: pd_md} induces a signed ``pseudo-QP'' tuple $(\bm{z}^*,\lambda_\mathcal{I}^*)$ that matches the 
    stationarity structure of the dist-QP \eqref{eq: dist_md} under a single active constraint. The pseudo-QP mapping is given by: 
    \begin{enumerate}
        \item Pseudo-QP multiplier: 
        $\lambda_\mathcal{I}^* = 2s^* \nu^*_{\mathcal{I}}$.\vspace{3pt}
        \item Pseudo-QP witness point and stationarity relation: 
        
        $\bm{z}^*=s^* \frac{\aciT}{\|\aci\|_2}=\frac{1}{2}\aciT\lambda_\mathcal{I}^*.$ 
    \end{enumerate}
\end{lemma}
\begin{proof}
    We first recall the stationarity condition of the dist-QP \eqref{eq: dist_md},
    which serves as the algebraic template for the pseudo-QP mapping:
    \begin{align}
        \bm{z}^* = -\frac{1}{2}\AciT (\lambdai)^\mathrm{QP}. \label{eq: qp-stat-cond}
    \end{align}
    Combining it with its active constraint $\Aci\bm{z}^*=\bci$, we solve for $(\lambdai)^\mathrm{QP} = -2\left[ \Aci\AciT \right]^{-1}\bci$, and substituting it back into \eqref{eq: qp-stat-cond} yields $\bm{z}^* =  \AciT \left[ \Aci\AciT \right]^{-1}\bci$,
    which, under a single active constraint, reduces to $\bm{z}^*=\frac{\bcci}{\|\aci\|^2_2}\aciT$,
    the same as the witness-point recovery formula in \eqref{eq: pd_zstar}.
    Next, since $\bm{x}\notin\mathcal{X}_\mathcal{S}$, the depth-LP \eqref{eq: pd_md} has a single active constraint,
    whose stationarity and active-constraint conditions give:
    \begin{subequations}\begin{align}
        &-1+{\nu}_\mathcal{I}^*\| \aci \|_2 =0\,\Rightarrow\,\nu^*_{\mathcal{I}} = 1 / \|\aci\|_2, \label{eq: pd-nu*}\\
        & \|\aci\|_2s^*-\bcci=0\,\Rightarrow\,s^* = \bcci / \|\aci\|_2.\label{eq: pd-s*}
    \end{align}\label{eq: pd-nu*-s*}
    \end{subequations}
    Substituting \eqref{eq: pd-s*} into  \eqref{eq: pd_zstar} yields
    \begin{align}
        \bm{z}^* =
        \frac{\bcci}{\|\aci\|_2}\frac{\aciT}{\|\aci\|_2} =s^* \frac{\aciT}{\|\aci\|_2}.
        \label{eq: z*-s*}
    \end{align}
    Thus, $\bm{z}^*$ recovered from the depth-LP \eqref{eq: pd_md} via \eqref{eq: pd_zstar} has the same projection formula as that obtained from the dist-QP template under a single active constraint. 
    It remains to construct the pseudo multiplier. Define the LP-induced pseudo multiplier as 
    \begin{align}
        \lambda_\mathcal{I}^* :=
        2s^* {\nu}_\mathcal{I}^*=\frac{2s^*}{\| \aci \|_2},
        \label{eq: lambda*-nu*}
    \end{align}
    where the second equality follows from \eqref{eq: pd-nu*}. By combining \eqref{eq: pd-nu*-s*}, \eqref{eq: z*-s*}, and \eqref{eq: lambda*-nu*}, the pseudo-QP stationarity relation is established:
    \begin{align}
        \bm{z}^* =
        s^* \frac{\aciT}{\|\aci\|_2} = \frac{1}{2}\aciT\frac{2s^*}{\|\aci\|_2} = \frac{1}{2}\aciT \lambda_\mathcal{I}^*,
    \end{align}
    which completes the proof.
\end{proof}
For notational simplicity, we hereafter use
$\lambdai\in\mathbb{R}_{\geq 0}^{|\mathcal I|}$ as a unified
notation for the active multiplier: it denotes the ordinary QP dual multiplier in the safe case and the LP-induced pseudo-QP multiplier of Lemma~\ref{lem: pseudo-qp-mapping} in the penetration case.
With this convention, both branches satisfy
\begin{align}
    \bm{z}^*= -\frac{1}{2}\mathrm{sgn}(\mathrm{sd}(\mathcal{R},\mathcal{O}))\AciT\lambdai.
    \label{eq: signed-stat-cond}
\end{align}
In summary, the dist-QP \eqref{eq: dist_md} and depth-LP \eqref{eq: pd_md} 
are companion problems that provide a generalized distance measure. Using the pseudo-QP mapping in Lemma~\ref{lem: pseudo-qp-mapping}, both branches admit a common tuple $(\bm{z}^*,\lambdai, \Aci, \bci)$ a.e., regardless of collision status, which will be used in Sec.~\ref{sec: 5} to compute the SDF gradient within a unified active-constraint-based framework.
The overall computational pipeline is illustrated in Fig.~\ref{fig: flow-cvxopt}.
\input{fig-flow-chart}

\subsection{Differentiability of the SDF in SE(2)} 
\label{subsec: diff-analysis}
We now investigate the differentiability of the proposed exact SDF \eqref{eq: sd_as_2norm} with respect to the robot's state $\bm{x}$. We consider a polygonal robot with $\bm{x}=[\bm{p}^{\top},\theta]^{\top}\in\mathrm{SE}(2)$, where $\bm{p}=[x,y]^\top\in\mathbb{R}^2$ is the position and $\theta\in \mathbb{S}^1$ is the orientation. For notational simplicity, let $\Phi(\bm{x}) := \mathrm{sd}(\mathcal{R}(\bm{x}),\mathcal{O})$ and let $R(\theta)=\begin{bmatrix}\cos\theta &  -\sin\theta\\\sin\theta&  \cos\theta\\\end{bmatrix}$ be a rotation matrix.
\begin{lemma}
    The proposed SDF $\Phi$, evaluated in the (Euclidean) \mdspace, is locally Lipschitz with respect to $\bm{x}\in\mathrm{SE}(2)$.
\end{lemma}
\begin{proof}
    By Assump.~\ref{assp1}, $\cobs(\bm{x})$ is nonempty, compact, and convex for every $\bm{x}$. Define its support function as $\sigma_{\cobs(\bm{x})}(\bm{a})=\sup_{\bm{z}\in \cobs(\bm{x})}\bm{a}^\top\bm{z}$ \cite{Bertsekas.book.09}. Combining it with Lemma~\ref{lem: zstar-unique-qp} and \eqref{eq: pd_md} gives $\Phi(\bm{x})=\mathrm{sd}(\bm{0},{\cobs(\bm{x})})=-\min_{\|\bm{a}\|_2=1}\sigma_{\cobs(\bm{x})}(\bm{a})$.
    We parameterize the robot by its state as $\mathcal{R}(\bm{x})=\bm{p}+R(\theta)\mathcal{R}_0$, where $\mathcal{R}_0$ denotes its body-fixed frame representation.
    By the CO definition in \eqref{eq: co} and standard support function calculus \cite{Boyd.Vandenberghe.04}, $\sigma_{\cobs(\bm{x})}(\bm{a})=\sigma_\mathcal{O}(\bm{a})-\bm{a}^\top\bm{p}+\sigma_{-\mathcal{R}_0}(R(\theta)^\top\bm{a})$.
    Let $\rho:=\max_{\mathbf{r}\in{-\mathcal{R}_0}}\|\mathbf{r}\|_2$ be the radius of the smallest enclosing disk centered at the body-frame origin.
    Fix any $\bm{x}_0\in\mathrm{SE}(2)$ and let $\mathcal{N}$ be a sufficiently small neighborhood of $\bm{x}_0$. For any $\bm{x}_1,\bm{x}_2\in\mathcal{N}$, define $\Delta\bm{p}=\bm{p}_1-\bm{p}_2$ and $\Delta\theta=\theta_1-\theta_2$.
    Then, for any unit vector $\bm{a}$: 
    \begin{align}\begin{split}
        &|\sigma_{\cobs(\bm{x}_1)}(\bm{a})-\sigma_{\cobs(\bm{x}_2)}(\bm{a})| 
        \\&=|-\bm{a}^\top\Delta\bm{p}+\sigma_{-\mathcal{R}_0}(R(\theta_1)^\top\bm{a})-\sigma_{-\mathcal{R}_0}(R(\theta_2)^\top\bm{a})|
        \\&\leq\|\Delta\bm{p}\|_2+\rho\|R(\theta_1)^\top\bm{a}-R(\theta_2)^\top\bm{a}\|_2\\
        &=\|\Delta\bm{p}\|_2+2\rho\left|\sin\left(\frac{\Delta\theta}{2}\right)\right|   \leq\|\Delta\bm{p}\|_2+\rho|\Delta\theta|.
    \end{split}\end{align}
    Since this bound is independent of $\bm{a}$, it holds uniformly over all unit directions. Taking the minimum therefore gives  $|\Phi(\bm{x}_1)-\Phi(\bm{x}_2)|\leq\|\Delta\bm{p}\|_2+\rho|\Delta\theta|\leq\sqrt{1+\rho^2}\|\bm{x}_1-\bm{x}_2\|_2$.
    This completes the proof.
\end{proof}
Since $\Phi$ is locally Lipschitz, it is differentiable a.e. \cite{Cortes.08.CSM}. We next recall a useful notion from nonsmooth analysis.

\begin{defi}[Critical Point \cite{Cortes.08.CSM}]
    A critical point of $h:\mathbb{R}^n \to\mathbb{R}$ is a point $\mathbf{x} \in \mathbb{R}^n$ such that $\bm{0}\in\partial h(\mathbf{x})$.
    \label{def: crit-pt}
\end{defi}
By Definition~\ref{def: crit-pt}, the maximizers and minimizers of a locally Lipschitz function are critical points. We analyze the differentiability properties of the SDF on SE(2) by first examining translation in $\mathbb{R}^2$ and then rotation in SO(2). 
\input{fig-nondiff}
\subsubsection{Translational Component in $\mathbb{R}^d$}
For fixed $\theta$, $\Phi(\bm{x})$ reduces to the standard Euclidean SDF in the translational variable $\bm{p}$, whose differentiability properties are well established in the literature, e.g., \cite[Chapter~2.3]{Osher.etal.book04} and \cite[Corollary 3.4.5]{Cannarsa.Sinestrari}. We know from Lemma~\ref{lem: sdf-prop} that $\nabla_{\bm{p}}\Phi(\bm{x})$ exists whenever the witness point $\bm{z}^*(\bm{x})$ is unique. It satisfies the Eikonal equation $\|\nabla_{\bm{p}}\Phi(\bm{x})\|=1$ and the normal-vector property a.e., which we will prove later in Sec.~\ref{subsec: sens-analysis-poly-bnd}. In our formulation, uniqueness always holds in the safe case by Lemma~\ref{lem: zstar-unique-qp}, and thus $\Phi$ is continuously differentiable with respect to $\bm{p}$. In the penetration case, uniqueness is equivalent to the depth-LP \eqref{eq: pd_md} having a single active constraint, i.e., $\bm{x}\notin\mathcal{X}_\mathcal{S}$. 

For $\bm{x}\in\mathcal{X}_\mathcal{S}$ where the translational gradient is undefined, we instead use the generalized gradient defined in \eqref{eq: gen_grad_h}. If the depth-LP~\eqref{eq: pd_md} has $\mathfrak{m}$ active constraints with unit normals $\{\bm{n}_1,\ldots,\bm{n}_\mathfrak{m}\}$, then $\partial_{\bm{p}}\Phi(\bm{x})=\mathrm{co}\{\bm{n}_1,\ldots,\bm{n}_\mathfrak{m}\}$. 
If $\bm{0}\in\partial_{\bm{p}}\Phi(\bm{x})$, 
then $\bm{x}$ is a critical point of $\Phi$ by Definition~\ref{def: crit-pt}. In the NCBF context, such critical points can obstruct safety recovery under pure translation, since no local increase of $\Phi$ can be certified from the generalized gradient; see Sec.~\ref{sec: 5}.
A closely related analysis can be found in \cite[Example~16]{Cortes.08.CSM}.
Figure~\ref{fig: non-diff}(a) illustrates a non-differentiable but non-critical state $\bm{x}_2$ and a critical state $\bm{x}_3$ whose active normals point in opposite directions so their convex hull contains the origin.

\subsubsection{Rotational Component in SO(2)}
By Lemma~\ref{lem: properties}, we know that each edge of the CO is a translated edge from either $\mathcal{R}$ or $\mathcal{O}$ (Fig.~\ref{fig: msum-demo}). A geometric degeneracy of the CO occurs when a robot-sourced edge and an obstacle-sourced edge merge into a single edge, which happens when an edge of $\mathcal{R}$ becomes parallel to an edge of $\mathcal{O}$ in \wspace. We denote the corresponding orientation by $\theta^*_{ij}$,
where the $i$-th edge of $\mathcal{R}$ is parallel to the $j$-th edge of $\mathcal{O}$. There exist $O(\ell_r\times\ell_o)$ such $\theta^*_{ij}$ for convex polygons of $\mathcal{R}$ and $\mathcal{O}$, and we denote their set by $\Theta^*=\{\theta^*_{ij}\mid i=1,\ldots,\ell_r,~j=1,\ldots,\ell_o\}$ \cite{Lozano.TC83}. Across $\theta^*_{ij}$, the edge ordering of the CO changes; see Fig.~\ref{fig: non-diff}(b).
Such angles are potential nonsmooth points of $\Phi$ with respect to $\theta$, but are actual non-differentiable points only when the active feature is affected by the combinatorial change.

For notational simplicity, we suppress the subscript $ij$, and denote the angles immediately before and after $\theta^*$ by $\theta^*_-$ and $\theta^*_+$, respectively. We take counter-clockwise (CCW) rotation as positive. In the illustrated clockwise (CW) rotation in Fig.~\ref{fig: non-diff}(b), we have $\theta^*_- > \theta^* > \theta^*_+$. The one-sided derivatives of $\Phi$ with respect to $\theta$ at the adjacent differentiable configurations have opposite signs: negative at $\theta^*_-$ and positive at $\theta^*_+$. Consequently, 
$0\in\partial_\theta\Phi(\bm{p},\theta^*)=\mathrm{co}\left\{\frac{\partial\Phi}{\partial(\theta)}\big\vert_{(\bm{p},\theta^*_-)}, \frac{\partial\Phi}{\partial(\theta)}\big\vert_{(\bm{p},\theta^*_+)}\right\}$.
Thus, in this illustrated case, with $\bm{p}$ fixed, $\theta^*$ is a critical point with respect to $\theta$ by Definition~\ref{def: crit-pt}.
It corresponds to a local maximizer of $\Phi$, i.e., the orientation providing a locally maximal minimum distance between the two sets. For local-gradient-based methods such as CBFs, such critical configurations may induce issues when coupled with system dynamics, as discussed later in Sec.~\ref{sec: local-min-analysis}.
\input{fig-msum-demo}

\section{SDF-Based NCBF Framework in \mdspace}\label{sec: 5}
This section presents our main result of the construction of NCBFs based on the {\it exact} SDF in \mdspace, together with a unified framework for computing its {\it exact} gradient on each locally fixed active branch via the optimal QP tuple (\ref{subsec: NCBF_grad_calc}).
We then exploit the geometric structure of 2D Minkowski operations to derive active-constraint sensitivities and the resulting closed-form gradient (\ref{subsec: sens-analysis-poly-bnd}). Finally, we present the controller synthesis (\ref{subsec: clf-cbf-qp}). Given the SDF in \mdspace~defined in \eqref{eq: sd_as_2norm}, we propose the following NCBF candidate, termed SD-NCBF: 
\begin{align}\begin{split}   
    h(\bm{x})&=\mathrm{sd}(\bm{0},\cobs(\bm{x}))-d_\mathrm{safe}\\
    &=\mathrm{sgn}(\mathrm{sd}(\bm{0},\cobs(\bm{x})))\|\bm{z}^*\|_2-d_\mathrm{safe}
    \label{eq: sd-ncbf},
\end{split}
\end{align}
where $\mathrm{sgn}(\cdot)$ denotes the sign function, accounting for the sign reversal between safe and colliding cases, and $d_\mathrm{safe}\geq0$ is a user-defined safety margin. Recall that both the witness point $\bm{z}^*(\bm{x})$ and $\cobs(\bm{x})$ depend on the robot's (current) state $\bm{x}$.

\subsection{Unified Active-Constraint-Based Gradient Computation} \label{subsec: NCBF_grad_calc}
To enforce the NCBF constraint \eqref{eq: ncbf-set-valued-constr}, we first need to compute the gradient of $h(\bm{x})$ for all differentiable $\bm{x}$. However, the proposed NCBF in \eqref{eq: sd-ncbf} depends on the witness point $\bm{z}^*(\bm{x})$, which is obtained as the solution to an optimization problem rather than from an explicit algebraic expression.
That is, $\bm{z}^*(\bm{x})$ depends implicitly on the state $\bm{x}$ through the constraints defining $\cobs(\bm{x})$. For $\bm{z}^*(\bm {x})\neq\bm{0}$, we apply the chain rule to \eqref{eq: sd-ncbf} with the SDF defined in \eqref{eq: sd_as_2norm} to obtain:
\begin{align}\begin{split} 
        \frac{\partial h}{\partial \bm{x}} &=\frac{\partial h}{\partial \bm{z}^*}\frac{\partial \bm{z}^*}{\partial \bm{x}} =
        \mathrm{sgn}(\mathrm{sd}(\mathcal{R},\mathcal{O}))\frac{(\bm{z}^*)^\top}{\|\bm{z}^*\|_2}\frac{\partial \bm{z}^*}{\partial \bm{x}}.
    \label{eq: dhdx_ift}
\end{split}\end{align}
The contact case $\bm{z}^*(\bm{x})=\bm{0}$ is treated separately later in Theorem~\ref{thm: dhdx-bd}. The first term is the unit normal along the shortest-distance direction, whereas the second term, $\partial \bm{z}^*/\partial \bm{x}\in\mathbb{R}^{d\times n}$, describes how the witness point $\bm{z}^*$, lying on the boundary of the CO in \mdspace, changes in response to changes in the robot's state in \wspace. It can be evaluated using \eqref{eq: dXidx} in Theorem~\ref{thm: diffopt-ift} with the reduced Jacobian matrices in \eqref{eq: dGdXi} and \eqref{eq: dGdx}, as in our prior work \cite{Chen.etal.CDC25}. However, computing this sensitivity term involves the inverse of the KKT matrix \eqref{eq: dGdXi}. Although its special block structure can be leveraged to accelerate the matrix inversion \cite{Parker.etal.arXiv26}, we adopt an inversion-free approach using the value-function sensitivity result \eqref{eq: dLdx} in Theorem~\ref{thm: diffopt-ift} to obtain the following result.
\begin{coro}\label{coro: dhdx}
    Suppose $\bm{0}\notin\cobs(\bm{x})$,
    i.e., $\bm{z}^*\neq\bm{0}$, and the dist-QP satisfies the regularity conditions in Theorem~\ref{thm: diffopt-ift}. Let $h$ be the SD-NCBF defined in \eqref{eq: sd-ncbf}, and $\bm{z}^*$ be the solution of the dist-QP \eqref{eq: dist_md} with objective function $f_0(\bm{z})=\|\bm{z}\|_2^2$. Then, at any differentiable state $\bm{x}$, the gradient of $h(\bm{x})$ is given by:
    \begin{align}\begin{split}
        \frac{\partial h}{\partial \bm{x}} &=\frac{\partial h}{\partial f_0(\bm{z}^*)}\frac{\partial f_0(\bm{z}^*)}{\partial \bm{x}} =\frac{1}{2\|\bm{z}^*\|_2}\frac{\partial\mathcal{L}}{\partial \bm{x}}\bigg\rvert_{(\bm{z}^*,\boldsymbol{\lambda}_\mathcal{I}^*)}\\
        &=\frac{\lambdaT}{2\|\bm{z}^*\|_2}\left(\frac{\partial \Aci}{\partial\bm{x}}\bm{z}^*-\frac{\partial \bci}{\partial\bm{x}}  \right).
        \label{eq: dhdx-dL}
    \end{split}\end{align}
\end{coro}

\begin{prop}\label{prop: gradient-equivalence}
    The gradient computed using the sensitivity of $\bm{z}^*$ via the IFT \eqref{eq: dhdx_ift} is equivalent to that computed using the partial derivative of the Lagrangian \eqref{eq: dhdx-dL}.
\end{prop}
\begin{proof}
    See Appendix~\ref{app: pf-ift-partialL}.
\end{proof}

\noindent We now show that the depth-LP \eqref{eq: pd_md}, through the pseudo-QP mapping in Lemma~\ref{lem: pseudo-qp-mapping}, yields the same active-constraint-based gradient form as the dist-QP branch \eqref{eq: dhdx-dL}.
%
\begin{thm}
    Suppose that $\bm{x}$ satisfies $\bm{0}\in\mathrm{int}(\cobs(\bm{x}))$ and $\bm{x}\notin\mathcal{X}_{\mathcal S}$, and that the active set $\mathcal{I}$ remains unchanged in a neighborhood of $\bm{x}$. Let $(\bm{z}^*, \lambda_{\mathcal{I}}^*)$ denote the pseudo-QP tuple derived from the optimal solution $(s^*,\nu_{\mathcal{I}}^*)$ of the depth-LP \eqref{eq: pd_md} via Lemma~\ref{lem: pseudo-qp-mapping}. Then, the gradient of the SD-NCBF $h(\bm{x})$ defined in \eqref{eq: sd-ncbf} is also governed by 
    \eqref{eq: dhdx-dL}.
    \label{thm: uni-grad}
\end{thm}
\begin{proof}
    Suppose $\bm{0}\in\mathrm{int}(\cobs(\bm{x}))$.
    We first rewrite the depth-LP \eqref{eq: pd_md} in the minimization form as $\min_s\,\,-s\,\,\,\,\mathrm{s.t.}\,\,
        s\|\mathbf{a}^c_{i}(\bm{x})\|_2  \leq {b}^c_{i}(\bm{x}),\,\, i=1,\ldots,\ell_c$,
    whose objective function and Lagrangian are $f_0^\mathrm{LP}(s)=-s$ and $\mathcal{L}^\mathrm{LP}(s,\bm{\nu})=-s+\sum_i\nu_i(s\|\mathbf{a}^c_{i}(\bm{x})\|_2  - {b}^c_{i}(\bm{x}))$, respectively. By Theorem~\ref{thm: diffopt-ift} and Remark~\ref{remark: LP-dL}, under the locally fixed active-set condition, we have  $\nabla_{\bm{x}}f_0^{\mathrm{LP}}(s^*(\bm{x}))=\frac{\partial\mathcal{L}^{\mathrm{LP}}}{\partial\bm{x}}\Big\vert_{(s^*,\nu_{\mathcal{I}}^*)}$ for all $\bm{x}\notin\mathcal{X}_\mathcal{S}$, which gives
    \begin{align}\begin{split}
        -\frac{\partial s^*}{\partial\bm{x}}
        &=\nu_{\mathcal{I}}^*\left(s^*\frac{\aci}{\|\aci\|_2}\frac{\partial\aciT}{\partial\bm{x}}-\frac{\partial\bcci}{\partial\bm{x}}\right).
        \label{eq: pf-lp-dLdx}
    \end{split}\end{align}
    Substituting the pseudo-QP mapping given by Lemma~\ref{lem: pseudo-qp-mapping}
    into \eqref{eq: pf-lp-dLdx} yields:
    \begin{align}
        -\frac{(\bm{z}^*)^\top}{\|\bm{z}^*\|_2}\frac{\partial \bm{z}^*}{\partial \bm{x}}=\frac{\lambda^*_{\mathcal{I}}}{2\|\bm{z}^*\|_2}\left((\bm{z}^*)^\top\frac{\partial\aciT}{\partial\bm{x}}-\frac{\partial\bcci}{\partial\bm{x}}\right).
        \label{eq: pf-lp-recover-qp}
    \end{align}
    Since $(\bm{z}^*)^\top\frac{\partial\aciT}{\partial\bm{x}}=\frac{\partial\aci}{\partial\bm{x}}\bm{z}^*$,
    the right-hand side of \eqref{eq: pf-lp-recover-qp} exactly recovers \eqref{eq: dhdx-dL}. Furthermore, the left-hand side of \eqref{eq: pf-lp-recover-qp} equals the gradient derived via the chain rule in \eqref{eq: dhdx_ift} with $\mathrm{sgn}(\mathrm{sd}(\mathcal{R}, \mathcal{O})) = -1$, thus establishing the equivalence between \eqref{eq: dhdx_ift} and \eqref{eq: dhdx-dL} for the depth-LP. Therefore, the exact gradient computation for both safe (dist-QP) and penetration (depth-LP) cases is completely unified under the active-constraint-based form in \eqref{eq: dhdx-dL}. This completes the proof.
\end{proof}

\begin{thm} \label{thm: dhdx-bd}
    Suppose that $\bm{z}^*(\bm{x})=\bm{0}\in\mathcal{E}(\cobs(\bm{x}))$ for $\bm{x}\notin\mathcal{X}_\mathcal{S}$ and $\theta\notin\Theta^*$. Assume the associated single active constraint 
    remains locally fixed. Then the safe and colliding branches yield the same limiting gradient of $h$ at $\bm{x}$, given by
    \begin{align}
        \frac{\partial h}{\partial \bm{x}} 
        =\frac{-\aci}{\|\aci\|_2}\frac{\partial \bm{z}^*}{\partial \bm{x}}
        =\frac{-1}{\|\aci\|_2}\frac{\partial \bcci}{\partial\bm{x}}.
        \label{eq: dhdx-z*=0}
    \end{align}
\end{thm}
\begin{proof}
    At contact, $\|\bm{z}^*\|_2=0$, and both the QP multiplier and the LP-induced pseudo multiplier vanish, so \eqref{eq: dhdx_ift}-\eqref{eq: dhdx-dL} are undefined.
    We therefore evaluate the limits from the safe and colliding sides while keeping the active constraint locally fixed. On either regular side, transposing and normalizing the signed stationarity relation \eqref{eq: signed-stat-cond} gives $-\mathrm{sgn}(\mathrm{sd}(\mathcal{R},\mathcal{O}))\frac{(\bm{z}^*)^\top}{\|\bm{z}^*\|_2}=\frac{\lambda_i^*\aci}{2\|\bm{z}^*\|_2}$. Taking norms on both sides yields $\frac{\lambda_i^*}{2\|\bm{z}^*\|_2}=\frac{1}{\|\aci\|_2}$. Substituting these relations into \eqref{eq: dhdx_ift} gives the first equality. At contact, the term $\frac{\partial\aci}{\partial\bm{x}}\bm{z}^*$ in \eqref{eq: dhdx-dL} vanishes, and the multiplier ratio above then yields the second equality. This completes the proof.
\end{proof}

\subsection{Closed-Form Gradient via Polytopic Boundary Sensitivities}
\label{subsec: sens-analysis-poly-bnd}
To compute the gradient of SD-NCBF $h$ in \eqref{eq: sd-ncbf} via either \eqref{eq: dhdx_ift} or \eqref{eq: dhdx-dL}, we need to derive the active-constraint sensitivities $(\frac{\partial\Aci}{\partial{\bm{x}}},\frac{\partial\bci}{\partial{\bm{x}}})$.
Here, $\frac{\partial\Aci}{\partial{\bm{x}}}\in\mathbb{R}^{|\mathcal{I}|\times 2\times n}$
is a 3D tensor and $\frac{\partial\bci}{\partial{\bm{x}}}\in\mathbb{R}^{|\mathcal{I}|\times n}$
is a matrix; together, they describe how the active boundary of the CO varies with the robot state $\bm{x}$.
In \wspace, the robot's occupied set can be parameterized with an H-rep as
$\mathcal{R}(\bm{x})=\{\mathbf{y}\in\mathbb{R}^d\mid A_r(\bm{x})\mathbf{y}\leq b_r(\bm{x})\}$, where $A_r(\bm{x})=A_0R(\theta-\theta_0)^{\top}$, $b_r(\bm{x})= b_0+A_r(\bm{x})\bm{p}-A_0\bm{p}_0$, with $A_0,b_0$ defining the robot's shape at its initial state $\bm{x}_0$. 
If the H-rep $(\mathbf{A}^c(\bm{x}), \mathbf{b}^c(\bm{x}))$ of the CO could be directly expressed as a function of $(A_r(\bm{x}), b_r(\bm{x}))$ and the obstacle's H-rep $(A_o, b_o)$, computing its sensitivities with respect to $\bm{x}$ would be straightforward. However, as the CO is defined via the Minkowski difference, no such explicit algebraic relation exists between their H-reps, making this computation nontrivial.

To bypass this issue, we derive analytical active-constraint sensitivities by leveraging the geometric properties of the 2D M-sum, as established in Lemma~\ref{lem: properties} (Fig.~\ref{fig: msum-demo}). Although this formulation is only locally valid (except under pure translation), it captures the state dependence of the active constraints and suffices for gradient computation. For a comprehensive discussion on the global construction of the CO, we refer readers to \cite{4M}, whose geometric insights will also be utilized in our subsequent analysis.
To proceed with our derivation, we again decouple the kinematic analysis into two parts: translation in Euclidean space $\mathbb{R}^d$ (briefly introduced in prior work \cite{Chen.etal.CDC25}) and rotation on SO(2).

\subsubsection{Translation in Euclidean Space $\mathbb{R}^d$}
For a purely translating robot, its outward normals remain fixed. Since we only consider static obstacles, the outward normals of the resulting CO are also fixed. Suppose the initial CO is given by $\mathcal{CO}=\{\mathbf{y}\in\mathbb{R}^d\mid\mathbf{A}^{c_0}\mathbf{y}\leq\mathbf{b}^{c_0}\}$, and the robot translates by $\bm{p}$. The CO is thus translated in the opposite direction of the robot's motion by $-\bm{p}$
. The translated CO becomes $\mathcal{CO}'=\{\mathbf{y}\in\mathbb{R}^d\mid\mathbf{A}^{c_0}(\mathbf{y}-(-\bm{p}))\leq\mathbf{b}^{c_0}\}=\{\mathbf{y}\in\mathbb{R}^d\mid\mathbf{A}^{c_0}\mathbf{y}\leq\mathbf{b}^{c_0}-\mathbf{A}^{c_0}\bm{p}\}$. Taking the partial derivatives of these parameters with respect to $\bm{p}$ yields:
\begin{align}
    \frac{\partial \Aci}{\partial\bm{p}} = \bm{0}, \quad\frac{\partial \bci}{\partial\bm{p}}= -\Aci.
    \label{eq: CO-sens-t}
\end{align}
\noindent Combining \eqref{eq: CO-sens-t} with Corollary~\ref{coro: dhdx} and Theorem~\ref{thm: uni-grad}, we obtain the following result.
\begin{prop}[Spatial Gradient]
\label{prop: cbf_spatial_grad}
    The gradient of the SD-NCBF $h$ in \eqref{eq: sd-ncbf} with respect to $\bm{p}$, for $\bm{x}\notin\mathcal{X}_\mathcal{S}$ and $\theta\notin\Theta^*$, is given by: 
    \begin{align}
        \frac{\partial h}{\partial \bm{p}} 
        =
        \begin{cases}
            -\textcolor{black}{\mathrm{sgn}(\mathrm{sd}(\mathcal{R},\mathcal{O}))}(\bm{z}^*)^\top/\|\bm{z}^*\|_2,
            & \text{if } \bm{z}^*\neq \bm{0} \\[0pt]
            \aci/\|\aci\|_2
            , & \text{if } \bm{z}^*= \bm{0}
        \end{cases}.
        \label{eq: cbf_spatial_grad}
    \end{align}
    Both cases satisfy $\big\|\frac{\partial h}{\partial \bm{p}}\big\|_2 = 1$. Moreover, for $\bm{z}^*\in\mathcal{E}(\cobs)$, the spatial gradient reduces to the active CO unit (outward) normal $\frac{\aci}{\|\aci\|_2}$, consistent with the unit-norm properties in Lemma~\ref{lem: sdf-prop}.
\end{prop}

\begin{proof}
    Suppose that $\bm{x}\notin\mathcal{X}_\mathcal{S}$ and $\theta\notin\Theta^*$. For $\bm{z}^*\neq\bm{0}$, substituting the translational sensitivity \eqref{eq: CO-sens-t} into the unified gradient formulation \eqref{eq: dhdx-dL} and using the signed stationarity relation \eqref{eq: signed-stat-cond} gives the first case. For $\bm{z}^*=\bm{0}$, combining Theorem~\ref{thm: dhdx-bd} and \eqref{eq: CO-sens-t} gives the second case. 
    Both properties follow immediately, which completes the proof.
\end{proof}


\subsubsection{Rotation in SO(2)}
We know from Lemma~\ref{lem: properties} that the CO edges are translations of robot or obstacle edges, except at parallel-edge degeneracies. Thus, for a small rotation of the robot, the robot-sourced edges of the CO rotate along with the robot, while the obstacle-sourced edges do not rotate and are only displaced by the motion of the associated robot feature.
\input{fig-opt-4cases-v}

Although the optimization is solved in \mdspace, identifying the source of an active CO edge requires relating it back to the original robot-obstacle features (vertices (\texttt{v}) or edges (\texttt{e})) in \wspace.
In \wspace, the robot-obstacle witness-point pair $(\mathbf{r}^*,\mathbf{o}^*)$, given by \eqref{eq: dist_w}, has four possible feature combinations: \{\texttt{v}, \texttt{e}\}, \{\texttt{e}, \texttt{v}\},
\{\texttt{v}, \texttt{v}\},
\{\texttt{e}, \texttt{e}\}. In this notation, the first and second entries specify the robot and obstacle features associated with the witness points $\mathbf{r}^*$ and $\mathbf{o}^*$, respectively (Fig.~\ref{fig: cvxopt-4cases}).
Note that both of them depend on the robot's state.

By Lemma~\ref{lem: properties}, the vertex-edge combinations, \{\texttt{v}, \texttt{e}\} and \{\texttt{e}, \texttt{v}\}, result in a single active constraint, i.e., $\bm{z}^*\in\mathcal{E}(\cobs(\bm{x}))$.
The source of this constraint is determined by the edge feature in \wspace~
(Fig.~\ref{fig: cvxopt-4cases} top two rows). Conversely, the \{\texttt{v}, \texttt{v}\} combination results in two active constraints, implying $\bm{z}^*\in\mathcal{V}(\cobs(\bm{x}))$
(Fig.~\ref{fig: cvxopt-4cases} 3rd~row). In this case, depending on the local geometry, the two adjacent CO edges may both be robot- or obstacle-sourced, or consist of one from each. The final \{\texttt{e}, \texttt{e}\} combination is a degenerate case, corresponding to non-unique optimal pairs in \wspace, with the parallel edges merging into a single edge in \mdspace~(Fig.~\ref{fig: cvxopt-4cases} 4th~row). Such geometric degeneracy causes a source switch and possible non-differentiability of the SDF, as discussed in Sec.~\ref{subsec: diff-analysis}. 

The above feature-combination analysis is used only to identify the source of each active CO edge, with the ultimate goal of recovering the associated robot witness point $\mathbf{r}^*$ required in the following result.

\begin{prop}[Rotational Gradient]
\label{prop: cbf_rotation_grad}
    The gradient of the SD-NCBF $h$ in \eqref{eq: sd-ncbf} with respect to $\theta$, at any differentiable state $\bm{x}\notin\mathcal{X}_\mathcal{S}$ with $\theta\notin\Theta^*$, is given by:
    \begin{align}
        \frac{\partial h}{\partial \theta} 
        =\frac{\partial h}{\partial \bm{p}}\bm{\ell}^\perp,
        \label{eq: cbf_rotation_grad}
    \end{align}
    where $\bm{\ell} = \mathbf{r}^* - \bm{p}$ is the robot witness vector, and $\bm{q}^\perp:=(-q_2, q_1)$ denotes the $90^\circ$ CCW rotation of $\bm{q}=(q_1,q_2)$.
\end{prop}
\begin{proof}
    Recall that a rotating vector $\bm{q}=R(\theta)\bm{q}_0$ satisfies $\frac{\partial \bm{q}}{\partial\theta} =\bm{q}^\perp$. For each $i\in\mathcal{I}$, consider the corresponding active CO edge. A robot-sourced edge locally rotates about a $\theta$-invariant point $\bm{v}_{\mathcal{O}}^c=\mathbf{o}^*-\bm{p}$ in \mdspace~(Fig.~\ref{fig: rotation-deri} \textbf{(a)}). Hence, its normal rotates with $\theta$, giving $\frac{\partial \aci}{\partial\theta} = \left(\aci\right)^\perp$. Since the $\theta$-dependent part of $\bcci$ arises from this rotation relative to $\bm{v}_{\mathcal{O}}^c$, we obtain $\frac{\partial \bcci}{\partial\theta} =\left(\aci\right)^\perp\bm{v}_\mathcal{O}^c$.
    By contrast, an obstacle-sourced edge only translates with the rotating robot witness vector $\bm{\ell}$, while its normal remains fixed  (Fig.~\ref{fig: rotation-deri} \textbf{(b)}). Thus, $\frac{\partial \aci}{\partial\theta}= \bm{0}$. Since its offset varies with the projection $\aci(-\bm{\ell})$
    and $\frac{\partial \bm{\ell}}{\partial\theta} =\bm{\ell}^\perp$, we  obtain $\frac{\partial \bcci}{\partial\theta} = -\aci\bm{\ell}^\perp$.
    We next substitute these sensitivities into \eqref{eq: dhdx-dL}.

    First, consider $\bm{z}^*\neq \bm{0}$. Since $\bm{z}^*=\mathbf{o}^*-\mathbf{r}^*$ and $\bm{\ell} = \mathbf{r}^*-\bm{p}$, we have $\bm{v}_\mathcal{O}^c=\bm{z}^*+\bm{\ell}$. For a robot-sourced constraint, the term in parentheses in \eqref{eq: dhdx-dL} reduces to $-\left(\aci\right)^\perp\bm{\ell}=\aci\bm{\ell}^\perp$.
    For an obstacle-sourced constraint, this term likewise reduces to $\aci\bm{\ell}^\perp$. Stacking over the active set, \eqref{eq: dhdx-dL} gives $\frac{\partial h}{\partial \theta}=\frac{\lambdaT}{2\|\bm{z}^*\|_2} \Aci\bm{\ell}^\perp=-\mathrm{sgn}(\mathrm{sd}(\mathcal{R},\mathcal{O}))(\bm{z}^*)^\top/\|\bm{z}^*\|_2\bm{\ell}^\perp=\frac{\partial h}{\partial \bm{p}}\bm{\ell}^\perp$, where the second equality follows from the signed stationarity relation \eqref{eq: signed-stat-cond}, and the last equality follows from the first case of \eqref{eq: cbf_spatial_grad}.
    At contact, $\bm{z}^*=\bm{0}$, so $\bm{v}_\mathcal{O}^c=\bm{\ell}$ and both cases give $\frac{\partial \bcci}{\partial\theta} = -\aci\bm{\ell}^\perp$.
    Applying Theorem~\ref{thm: dhdx-bd} yields $\frac{\partial h}{\partial \theta}
    =\frac{\aci}{\|\aci\|_2}\bm{\ell}^\perp=
    \frac{\partial h}{\partial \bm{p}}\bm{\ell}^\perp$, where the last equality follows from the second case of \eqref{eq: cbf_spatial_grad}. This completes the proof.
\end{proof} 
We remark that, in the safe case, the gradients in \eqref{eq: cbf_spatial_grad}-\eqref{eq: cbf_rotation_grad} can also be evaluated from the \wspace~solution in \eqref{eq: dist_w}. The key advantage of the \mdspace~formulation is its unified treatment of both branches: $\bm{z}^*$ and its active CO feature allow recovery of $(\mathbf{r}^*,\mathbf{o}^*)$ and hence $\bm{\ell}$, whereas no convex penetration formulation is known in \wspace.
\input{fig-CO-rot-deri}

\begin{remark}[Extension to 3D]
    The source-based derivation of Proposition~\ref{prop: cbf_rotation_grad} relies on the 2D boundary structure of the CO (Lemma~\ref{lem: properties}), so it does not directly extend to 3D. In 3D, nonparallel edge pairs can generate new CO faces with normals inherited from neither set. A complete 3D characterization is left for future work.
\end{remark}

\subsection{\textcolor{black}{Controller Synthesis with CLFs via QPs}}\label{subsec: clf-cbf-qp}
Since our SD-NCBF $h$ is locally Lipschitz (differentiable a.e.), the standard CBF framework \cite{Ames.etal.ECC19} does not directly apply. As established in Theorem~\ref{thm: safety-guarantee}, safety is guaranteed by a set-valued inner-product condition that accounts for all possible pairings between the generalized gradient and the Filippov closed-loop dynamics.
Direct evaluation of this condition, however, can be computationally demanding. For real-time controller-synthesis purposes, the following result provides a sufficient condition for the SD-NCBF $h$ in \eqref{eq: sd-ncbf}, following from Theorem~\ref{thm: safety-guarantee} and paralleling the formulation in \cite{Glotfelter.etal.CCTA18}.
\begin{coro}[Sufficient condition for SD-NCBF]
    \label{coro: suff-SD-NCBF}
    Let $h$ be the SD-NCBF candidate defined in \eqref{eq: sd-ncbf} and $C\subset\mathcal{X}$
    be the corresponding nonempty safe set. For each $\bm{x}\in\mathcal{X}$, let $\mathcal{A}(\bm{x})$ denote the index set of local SDF branches $h_a$ active at $\bm{x}$, i.e., $h_a(\bm{x})=h(\bm{x})$ for all $a\in\mathcal{A}(\bm{x})$. Their gradients generate the generalized gradient $\partial h(\bm{x})=\mathrm{co}\{\nabla h_a(\bm{x})\mid a\in\mathcal{A}(\bm{x})\}$. 
    Suppose there exist 
    a locally Lipschitz extended class-$\mathcal{K}$ function $\alpha:\mathbb{R}\rightarrow\mathbb{R}$ and a measurable, locally bounded feedback controller $\bm{u}:\mathcal{X}\to\mathcal{U}$ such that, for every $\bm{x}\in\mathcal{X}$, there exists a neighborhood $\mathcal{N}_{\bm{x}}$ of $\bm{x}$ such that
    \begin{align}
        L_fh_a(\tilde{\bm{x}})+L_gh_a(\tilde{\bm{x}})\bm{u}(\tilde{\bm{x}})\geq-\alpha(h_a(\tilde{\bm{x}})),
        \label{eq: ncbf-constr}
    \end{align}
    for all $\tilde{\bm{x}}\in\mathcal{N}_{\bm{x}}$ and all $a\in\mathcal{A}(\bm{x})$. Then $h$ is a valid SD-NCBF, and $\bm{u}$ renders $C$ forward invariant.
\end{coro}
\begin{remark}\label{remark: almost-active}
    In our setting,
    $|\mathcal{A}(\bm{x})|>1$ occurs (i) at skeleton states $\bm{x}\in\mathcal{X}_\mathcal{S}$, where the depth-LP \eqref{eq: pd_md} has multiple active constraints associated with different CO-edge normals; (ii) when an active CO edge becomes degenerate at a critical angle $\theta\in\Theta^*$, yielding
    multiple limiting rotational sensitivities. In implementation, following \cite{Glotfelter.etal.CCTA18}, we enforce 
    \eqref{eq: ncbf-constr} over the almost-active set $\tilde{\mathcal{A}}(\bm{x})$, which augments $\mathcal{A}(\bm{x})$ with the indices of local SDF branches whose values are close enough to $h(\bm{x})$ to be treated as active.
\end{remark}
Next, we solve \mbox{Problem \ref{prob}} by combining a given CLF $V$ with the SD-NCBF $h$ in \eqref{eq: sd-ncbf} via a QP \cite{Ames.etal.ECC19}.
For implementation, we partition the horizon $[0,T]$ into equal intervals of $\Delta t>0$, and hold the state fixed at the start of each step. We assume the control is constant for each time step interval $[t_0+k\Delta t,t_0+(k+1)\Delta t)$ with $k\in\mathbb{Z}_{\geq0}$. 
At each time interval, the following two-stage procedure is executed to compute the control input $\bm{u}^*$. First, depending on the collision status, we solve either the dist-QP \eqref{eq: dist_md} or the depth-LP \eqref{eq: pd_md} to evaluate $h$, and then identify the almost-active set $\tilde{\mathcal{A}}(\bm{x})$. For each $a\in\tilde{\mathcal{A}}(\bm{x})$, we obtain the associated QP tuple directly or construct the pseudo-QP tuple using
Lemma~\ref{lem: pseudo-qp-mapping}, and then compute $\nabla h_a$ using \eqref{eq: cbf_spatial_grad}-\eqref{eq: cbf_rotation_grad}.
Second, we solve the following CLF-NCBF-QP:
\begin{align}
    \bm{u}^*(\bm{x})&=
    \underset{(\bm{u},\delta)\in\mathcal{U}\times\mathbb{R}}{\arg\min}\,\,
    \bm{u}^\top H\bm{u}+p\delta^2\notag\\[0.2em]
    \mathrm{s.t.} \quad&
    \,\,L_fV(\bm{x})+ L_gV(\bm{x})\bm{u}+\kappa V(\bm{x})\leq\delta,\label{eq: clf-ncbf-qp}
    \\
    &\,\,L_fh_a(\bm{x})+L_gh_a(\bm{x})\bm{u}\geq-\gamma h_a(\bm{x})+\epsilon,\,\,\forall a\in\tilde{\mathcal{A}}(\bm{x}),\notag
\end{align}
where $\kappa,\gamma>0$ are hyperparameters, $H\in\mathbb{S}^{q}_{++}$ and $p>0$ penalize control and CLF relaxation, respectively; $\tilde{\mathcal{A}}(\bm{x})$ is the almost-active index set introduced in Remark~\ref{remark: almost-active}. The parameter $\epsilon\geq0$ governs safety recovery: $\epsilon=0$ recovers the sufficient condition in Corollary~\ref{coro: suff-SD-NCBF} with $\alpha(s)=\gamma s$, while $\epsilon>0$ ensures finite-time recovery from the unsafe set, provided that \eqref{eq: clf-ncbf-qp} remains feasible and no unsafe critical point is encountered.
We refer interested readers to \cite{Li.etal.arXiv26} for a detailed discussion of this finite-time recovery guarantee.
\begin{figure*}[ht!]
    \centering
    \includegraphics[width=0.96\textwidth]{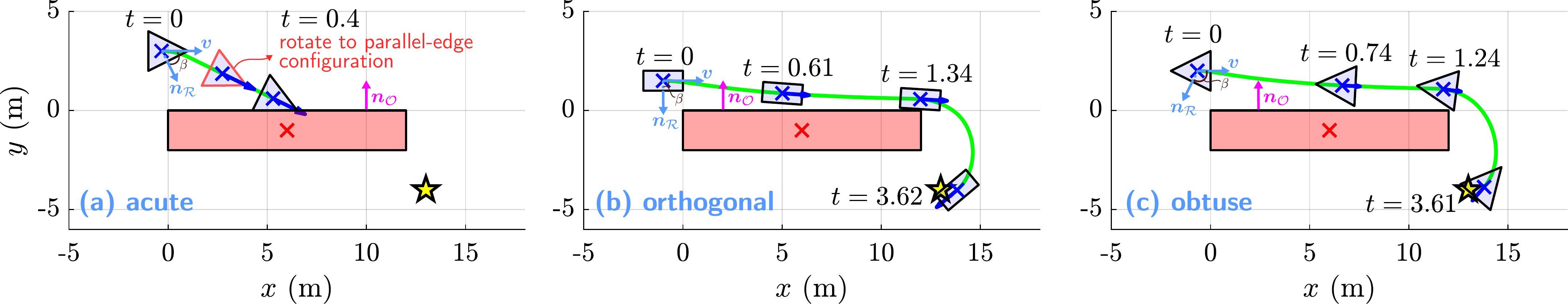}    
    \caption{Demonstration of GKC local minimum with three different shapes of unicycle robots. \textbf{(a) Acute $(\beta<\pi/2)$}: The local rotational gradient of SD-NCBF steers the robot into a parallel-edge configuration, locking its heading into the obstacle and leading to a GKC local minimum. \textbf{(b) Orthogonal $(\beta=\pi/2)$} and \textbf{(c) obtuse $(\beta>\pi/2)$}: The parallel alignment directs the heading tangential to or away from the obstacle. After balancing with goal-orienting CLFs, the robots reach the goal successfully. 
    } 
    \label{fig: geo-kine-local-min}
\end{figure*}

\section{Geometric-Kinematic-Coupled Local Minima}
\label{sec: local-min-analysis}
The exact rotational gradient in \eqref{eq: cbf_rotation_grad} reveals a counterintuitive undesired equilibrium arising from the coupling between the distance-based metric, exact robot geometry, and nonholonomic kinematics. We term this equilibrium the {\it geometric-kinematic-coupled} (GKC) local minimum. In our preliminary work \cite{Chen.etal.CDC25}, the rotational sensitivity (and the resulting gradient) was approximated in a way that underestimated the rotational effect; surprisingly, this numerical inaccuracy unintentionally masked this deadlock from our previous observations. 

Since the CLF-(N)CBF-QP formulation \eqref{eq: clf-ncbf-qp} is inherently local and reactive, undesired equilibria can still occur. For instance, deadlocks may arise even for point-mass robots at obstacle boundaries \cite{Reis.etal.LCSS21}. Accounting for finite robot geometry introduces additional traps, e.g., a square robot stalling beside a wide obstacle with the goal behind it.
Characterizing all such topological traps is beyond the scope of this paper; we focus instead on the newly identified GKC local minimum.

\subsection{Geometric Condition for GKC Susceptibility}
As discussed in Sec.~\ref{subsec: diff-analysis}, at a critical angle $\theta^*$, corresponding to the \{\texttt{e}, \texttt{e}\} configuration, the SDF is locally maximized with respect to rotation; see Fig.~\ref{fig: non-diff}\textbf{(b)}. Thus, when the SD-NCBF engages to prevent collision, the rotational gradient of $h$ can drive the robot toward this parallel-edge configuration. Although this alignment is locally desirable for increasing the safety margin, it can lock nonholonomic robots with certain shapes in the parallel-edge configuration since any sufficiently small rotation decreases $h$. If the resulting heading points toward the obstacle, the robot instead translates toward it and may eventually stall on its boundary,
as illustrated in Fig.~\ref{fig: geo-kine-local-min}\textbf{(a)}. 

To reason about this susceptibility geometrically, let $\mathcal{N}_\mathcal{R}$ and $\mathcal{N}_\mathcal{O}$ denote the candidate outward unit normals of the active robot and obstacle features, respectively, and define $\beta(\bm{n}_\mathcal{R})=\cos^{-1}\left(\frac{\bm{n}_\mathcal{R}\cdot\bm{v}}{\|\bm{v}\|_2} \right)$ for $\bm{n}_\mathcal{R}\in\mathcal{N}_\mathcal{R}$. A feature pair satisfying $\beta(\bm{n}_\mathcal{R})<\pi/2$ and $\bm{n}_\mathcal{R}^\top\bm{n}_\mathcal{O}<0$ is susceptible to GKC trapping. The mechanism underlying this susceptibility can be seen from the heading velocity at the resulting critical configuration, where $\bm{n}_\mathcal{R}=-\bm{n}_\mathcal{O}$. Since the acute $\beta$ implies $\bm{n}_\mathcal{R}^\top\bm{v}>0$, it follows that $\bm{n}_\mathcal{O}^\top\bm{v}=-\bm{n}_\mathcal{R}^\top\bm{v}<0$. Hence, at the critical configuration, the heading-constrained velocity $\bm{v}$ points inward relative to the active obstacle normal, explaining the susceptibility to GKC trapping.

\subsection{Case Study}
We now elaborate on the aforementioned geometric condition through a case study with the same goal, obstacle, hyperparameters, and control limits, but three different convex polygonal unicycle robots. 
The robot with an acute $\beta$ is trapped in the GKC local minimum as expected (Fig.~\ref{fig: geo-kine-local-min}\textbf{(a)}). In contrast, for the orthogonal (Fig.~\ref{fig: geo-kine-local-min}\textbf{(b)}) and the obtuse (Fig.~\ref{fig: geo-kine-local-min}\textbf{(c)}) robots, the same rotational tendency would make the heading tangent to or directed away from the obstacle at the critical alignment, respectively, thereby allowing both to avoid the GKC local minimum. In the closed-loop CLF-NCBF-QP trajectories, this tendency is balanced by the CLF before the critical configuration is reached.

We note that the choice of SD-NCBF hyperparameter $\gamma$ in \eqref{eq: clf-ncbf-qp} can also induce a similar GKC rotation, even when the robot does not ultimately become trapped. Specifically, a smaller $\gamma$ activates the NCBF constraint at a greater safe distance, causing the robot to initiate the counterintuitive evasive maneuver of turning toward the obstacle. This leads to significantly different trajectories even when all other hyperparameters remain fixed; refer to Fig.~\ref{fig: simu-effect-gamma}. A comprehensive characterization of how the GKC mechanism interacts with the CLF-NCBF-QP and its hyperparameters, together with systematic mitigation strategies, is left for future work.
\begin{figure}[hbt!]
    \centering\includegraphics[width=0.98\linewidth]{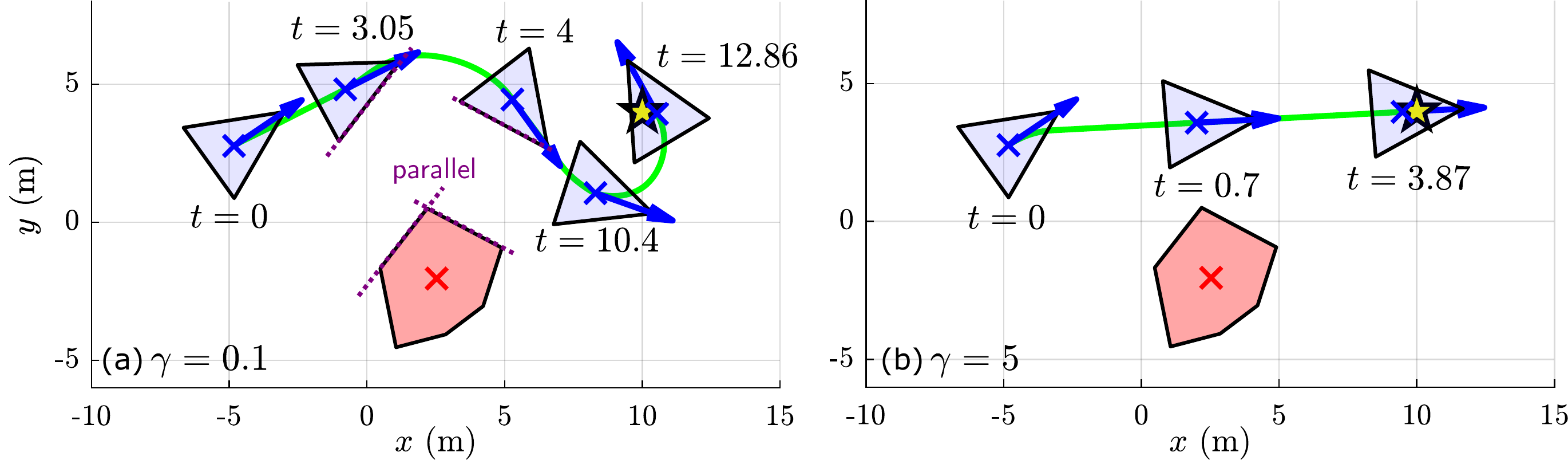}
    \caption{The effect of the SD-NCBF hyperparameter $\gamma$ in \eqref{eq: clf-ncbf-qp} on GKC-induced rotations. \textbf{(a)} A smaller
    $\gamma$ induces early obstacle-aware rotation.
    \textbf{(b)} A larger $\gamma$ delays activation and 
    avoids unnecessary rotations.}
    \label{fig: simu-effect-gamma}
\end{figure}


\section{\textcolor{black}{Simulation Results}}
We demonstrate the proposed SD-NCBF \eqref{eq: sd-ncbf}, along with the designed CLFs, by solving \mbox{Problem \ref{prob}} via QPs \eqref{eq: clf-ncbf-qp} in four case studies. We first consider the robot with single-integrator dynamics (\ref{subsec: si}) for single-obstacle avoidance. Second, we consider the robot with unicycle dynamics (\ref{subsec: unicycle}) for avoidance involving a single obstacle (\ref{subsub: single-obs}).
We then present the special case of starting in collision to highlight the safety recovery capability of our unified SDF formulation (\ref{subsub: unsafe-init}). Lastly, we demonstrate the non-conservative navigation in a maze-like environment enabled by our exact SD-NCBF formulation (\ref{subsub: multi-obs}).
All simulations are conducted in MATLAB, using \texttt{quadprog} to solve both the dist-QP \eqref{eq: dist_md} and the CLF-NCBF-QP \eqref{eq: clf-ncbf-qp}, \texttt{linprog} to solve the depth-LP \eqref{eq: pd_md}, and \texttt{ode45} to integrate the system dynamics with a loop rate of 200 Hz, on a PC with an Intel i7-8700 6-core of 3.2 GHz CPU and 32 GB of RAM. See video at \normalfont{\texttt{\url{https://youtu.be/D0zVswzyxaE}}}.

\subsection{Single-Integrator Model (Pure Translation)}\label{subsec: si}
We first consider a diamond-shaped robot translating to the goal while avoiding a polygonal obstacle. The state and dynamics are given by $\bm{x} =[x,y]^{\top},\,\,\dot{\bm{x}}=\bm{u}=[u_1,u_2]^{\top}$, where $x,y$ denote the position coordinates, and $u_1,\,u_2$ are linear velocity control inputs. In the form of \eqref{eq: ctrl-affine}, we have $f(\bm{x})=\bm{0}_{2\times 1}$ and $g(\bm{x})=\bm{I}_2$. A CLF $V(\bm{x})=\|\bm{x}-\bm{x}_g\|_2^2$ is used to reach $\bm{x}_g=(7,1)$. The other parameters are $\bm{x}(0)=(1,7),\,\bm{u}_{\max}=-\bm{u}_{\min}=5$~m/s, $d_\text{safe}=0,\,p=8,\,\kappa=2,\,\gamma = 5$, $\epsilon=0$. 
At each control loop, we first solve the convex program for the SDF $h$, then solve the CLF-NCBF-QP \eqref{eq: clf-ncbf-qp} for $\bm{u}^*(\bm{x})$. We compare our method with \cite{Molnar.CCTA25}, which utilizes vertices to construct CBFs based on an under-approximated SDF via LSE. The simulation results (Fig.~\ref{fig: simu-pure-trans}) demonstrate that our exact SD-NCBF outperforms \cite{Molnar.CCTA25} by generating a less conservative trajectory (green versus dashed magenta) with more straightforward safety margin tuning. The runtime of \cite{Molnar.CCTA25} is almost half that of ours, as its computation of $h$ only involves algebraic operations on the obstacle's vertices, while ours requires computing the CO\footnote{We note that for pure translation cases, the CO can be parameterized by position and computed only once, although our implementation does not exploit this property to improve runtime.} and solving a convex program for $h$. However, this computational efficiency comes at the cost of increased conservativeness and sensitive parameter tuning, as an insufficient buffer may cause the robot to slightly clip the obstacle; see the purple robot at $t=0.88$.
\begin{figure}[hbt!]
    \centering\
        \begin{xy}
            \xyimport(100,100){\includegraphics[width=0.98\linewidth]{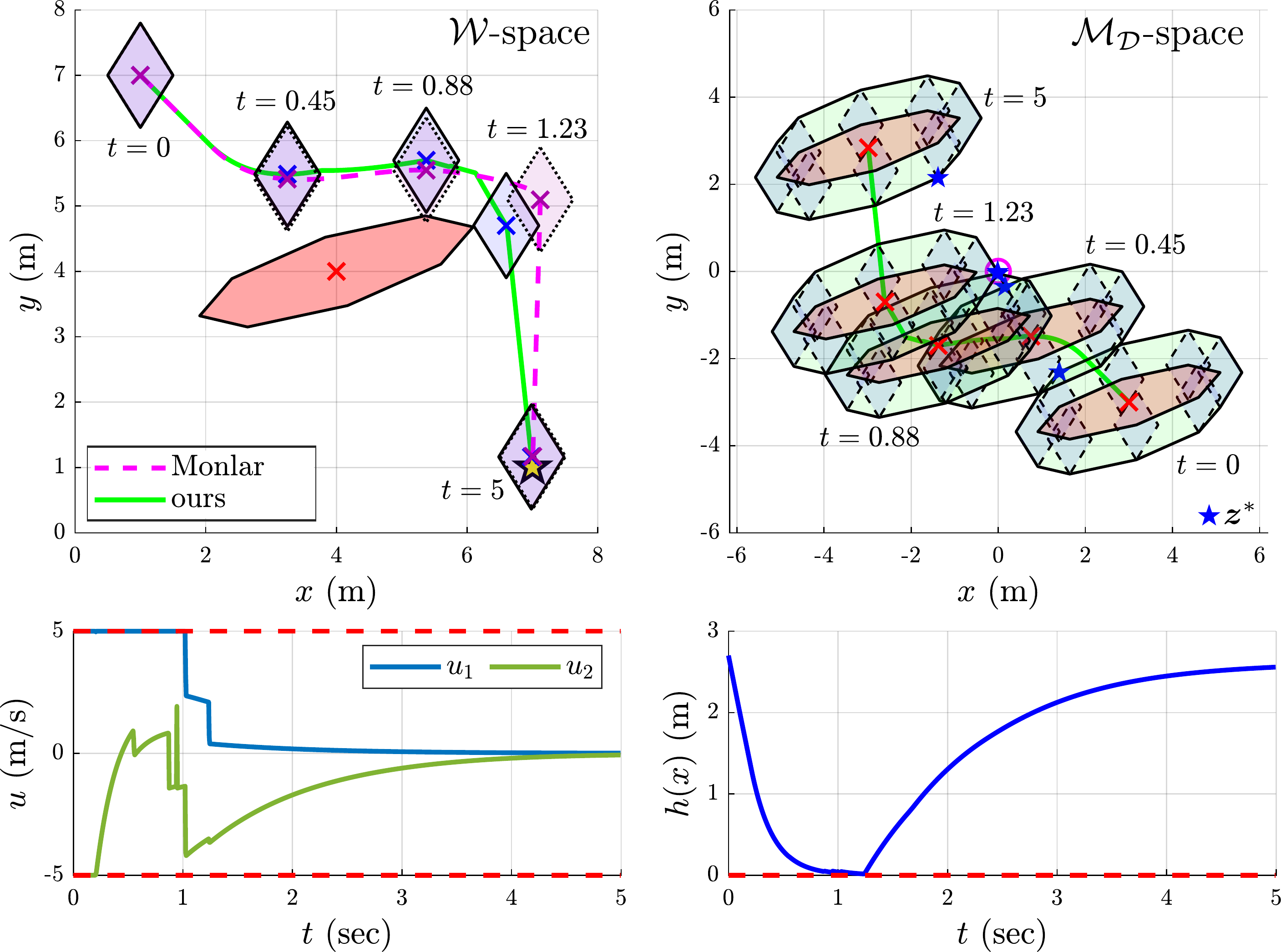}}
            ,(22.62,50.85)*{\text{\tiny\cite{Molnar.CCTA25}}}
        \end{xy}
    \caption{\textbf{Top-Left:} Snapshots show the blue (ours) and purple (\cite{Molnar.CCTA25}) robots avoiding the obstacle (red) with a tighter trajectory (green) than the baseline (dashed magenta). 
    \textbf{Top-Right:} The corresponding right-to-left motion of our CO (green) in \mdspace. \textbf{Bottom:} Control inputs $\bm{u}(t)$ and NCBF $h(\bm{x})$ satisfy actuator and safety constraints.}
    \label{fig: simu-pure-trans}
\end{figure}

\subsection{Unicycle Model}\label{subsec: unicycle}
Now we consider the robot with unicycle dynamics given by ${\dot{x}=v\cos{\theta}}$, ${\dot{y}=v\sin{\theta}}$, and ${\dot{\theta}=\omega}$, with state $\bm{x}=[\bm{p}^\top,\theta]^{\top}$ and control $\bm{u}=[u_1,u_2]^{\top}=[v,\omega]^{\top}$, where $v, \omega$ are linear and angular velocities, respectively. In the form of \eqref{eq: ctrl-affine}, we have $f(\bm{x})=\bm{0}_{3\times1}$ and $g(\bm{x})=[\cos\theta~0;\sin\theta~0; 0~1]$. 
We remark that, unlike the obstacle in \wspace, the CO depends on both the robot position and orientation; therefore, $h$ has relative degree one under the unicycle dynamics.
We use a CLF $V(\bm{x})=\frac{1}{2}(\Delta x^2+\Delta y^2+\Delta\theta^{2})$ with $\Delta x=x_d-x$, $\Delta y=y_d-y$, and $\Delta\theta=\theta-\operatorname{atan2}(\Delta y,\Delta x)$ 
to reach the goal. To prevent the GKC deadlock introduced in Sec.~\ref{sec: local-min-analysis}, we utilize a rectangular robot throughout the simulations. We evaluate its performance across the following three scenarios:
\subsubsection{Moving around a single obstacle} \label{subsub: single-obs}
In this case, we compare our method with \cite{Thirugnanam.etal.ACC22}, which computes the MDF via duality-based optimization and imposes the CBF constraints using the lower bound of its gradient. The other parameters are $\bm{x}_0=(-2.5,3.75,\pi/9),\,\bm{x}_g=(10,10)$, $u_{1,{\max}}=-u_{1,{\min}}=10$~m/s, $u_{2,{\max}}=-u_{2,{\min}}=\pi/2$ rad/s, $d_\text{safe}=0,\,p=1,\,\kappa=1.5,\,\gamma = 5,\,\epsilon=0$. The simulation results show that our robot (blue) reaches the goal faster with a tighter trajectory (green); see Fig.~\ref{fig: simu-unicycle-comp}-Top. Furthermore, in the 6-gon case (Fig.~\ref{fig: simu-unicycle-comp}-Bottom), which involves more rotations and optimal feature switches, we observe that \cite{Thirugnanam.etal.ACC22} is more prone to infeasibility. In contrast, our method successfully reaches the goal while guaranteeing safety in both cases.
\begin{figure}[hbt!]
    \centering
    \begin{xy}
        \xyimport(100,100){\includegraphics[width=0.98\linewidth]{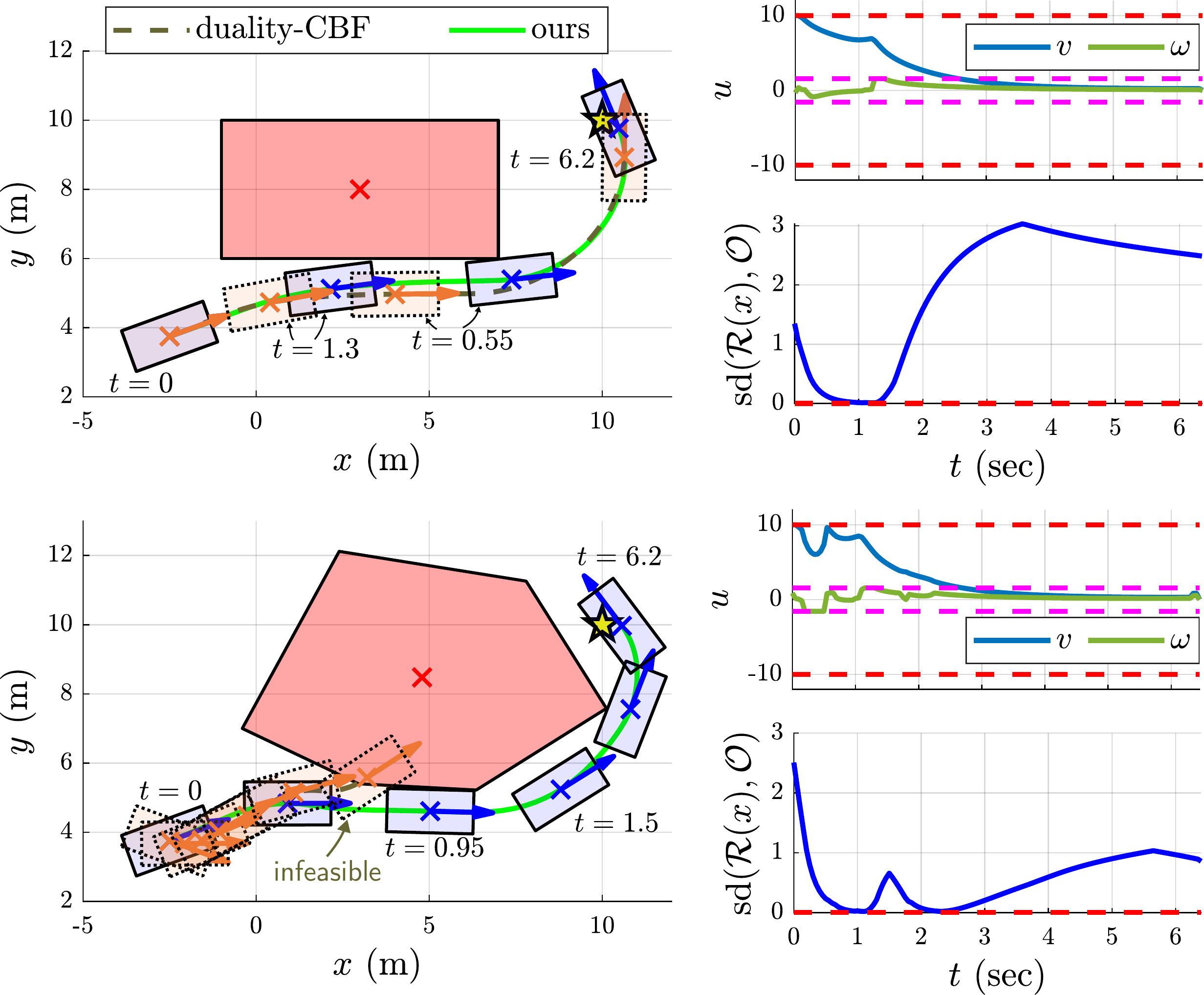}}
        ,(33.62,97.15)*{\text{\scriptsize \cite{Thirugnanam.etal.ACC22}}}
    \end{xy}
    \caption{\textbf{Left:} 
    Our SD-NCBF (blue) vs. the duality-MDF-CBF (yellow) \cite{Thirugnanam.etal.ACC22}  under unicycle dynamics. For the 4-gon obstacle (\textbf{Top}), our robot follows a less conservative trajectory (green) and reaches the goal faster; for the 6-gon obstacle (\textbf{Bottom}) involving more optimal feature switches, ours is significantly less prone to infeasibility than the baseline. \textbf{Right:} Profiles of the control input $\bm{u}(t)$ and CBF $h(\bm{x})$ for our method. 
    }
    \label{fig: simu-unicycle-comp}
\end{figure}

\subsubsection{Initializing in an unsafe set} \label{subsub: unsafe-init}
Because the safety metric is an SDF, the same formulation also applies inside the unsafe set. Theorem~\ref{thm: uni-grad} establishes that the dist-QP and depth-LP share the unified gradient formula through the pseudo-QP mapping of Lemma~\ref{lem: pseudo-qp-mapping}.
The parameters used in this case are $\bm{x}_0=(-3.75,2.55,\pi/5),\,\bm{x}_g=(8,10)$, $u_{1,{\max}}=-u_{1,{\min}}=10$~m/s, $u_{2,{\max}}=-u_{2,{\min}}=\pi$ rad/s, $d_\text{safe}=0,\,p=10,\,\kappa=2,\,\gamma = 1,\,\epsilon=0.2$. When the robot starts in an unsafe set, we prioritize safety recovery over goal reaching by disabling the CLF temporarily until the collision is resolved. The robot trajectory and value of the SDF are shown in Fig.~\ref{fig: simu-unicycle-unsafe}, which demonstrates that our SD-NCBF is able to bring the robot starting in the unsafe set back to the safe set.
\begin{figure}[hbt!]
    \centering\includegraphics[width=0.98\linewidth]{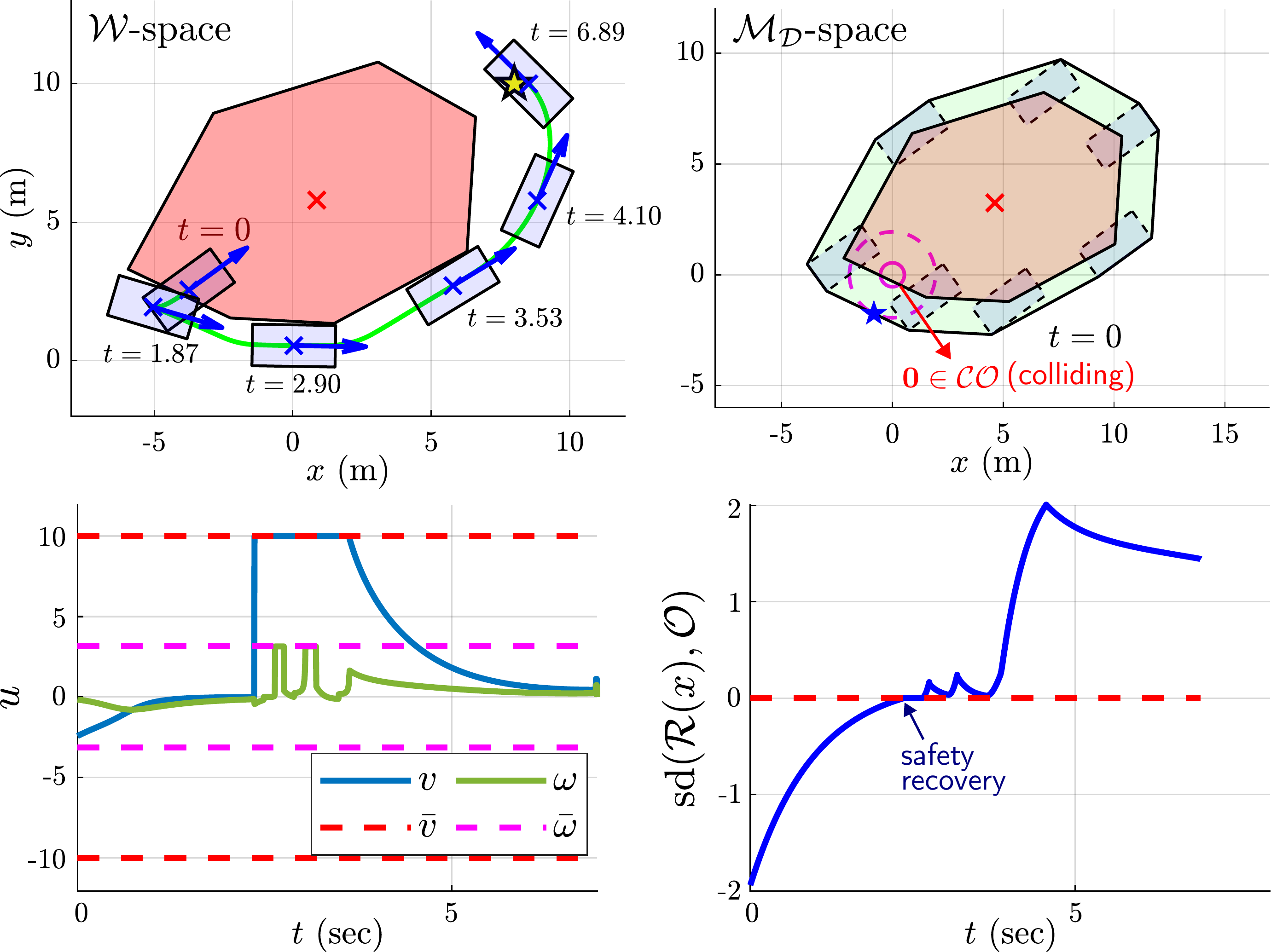}
    \caption{\textbf{Top}: The robot starts in collision with the obstacle in \wspace~(left), and the CO contains the origin in \mdspace~(right). The robot successfully resolves the collision
    using our SD-NCBF by time $t=2.5$. \textbf{Bottom-Left:} Control inputs $\bm{u}(t)$ satisfy actuator limits. \textbf{Bottom-Right:} The CBF $h(\bm{x})$ changes from negative to positive, corresponding to safety recovery. }
    \label{fig: simu-unicycle-unsafe}
\end{figure}

\subsubsection{Navigating in a maze-like environment} \label{subsub: multi-obs}
The maze-like environment consists of 5 convex polygonal obstacles. Safety is guaranteed by imposing NCBF constraints with respect to each obstacle $\mathcal{O}_j,\,j\in\{1,\ldots,5\}$. The parameters used are the same as in the previous setting, except for
$\bm{x}_0=(1,1,\pi/12),\,\bm{x}_g=(10,10.3),\,p=8,\,\kappa=2,\,\gamma = 10,\,\epsilon=0$. Our exact SD-NCBF framework allows the robot to navigate safely through a narrow passage while reaching the goal position (Fig.~\ref{fig: simu-multi-obs}). The time evolution of each CBF $h_j(\bm{x})$ (Fig.~\ref{fig: simu-multi-obs}-Right-Bottom) shows that all values remain positive over time and thus the robot is collision-free with respect to each obstacle throughout the trajectory.
\begin{figure}[hbt!]
    \centering\includegraphics[width=0.98\linewidth]{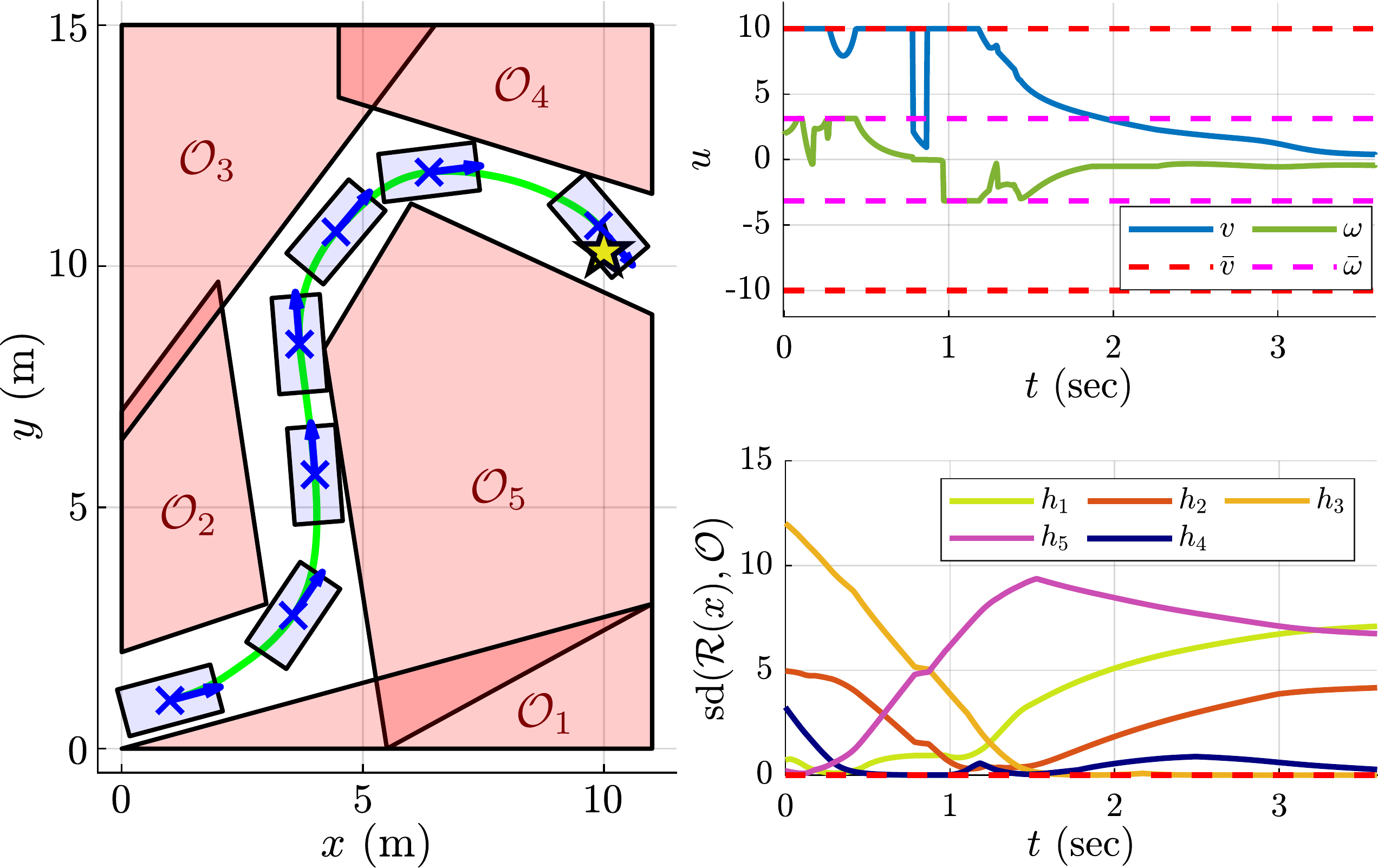}
    \caption{\textbf{Left:} Snapshots of robot motion in a maze-like environment in $\mathcal{W}$-space demonstrate the non-conservativeness of our SD-NCBF. \textbf{Right:} Control inputs $\bm{u}(t)$ and the CBFs $h_j(\bm{x}),~j=1,\ldots,5$, for each obstacle satisfy actuator and safety constraints over time.}
    \label{fig: simu-multi-obs}
\end{figure}

\subsection{Discussion and Limitations}
Our SD-NCBF framework has demonstrated its non-conservativeness compared to existing (N)CBF methods. The robot can navigate among obstacle(s)  {\it without using any predefined waypoints or reference paths}. However, as discussed in Sec.~\ref{sec: local-min-analysis}, since the CLF-(N)CBF-QP formulation is fundamentally a local method, deadlocks are inevitable. Considering the exact geometry of the robot and obstacles introduces a broader class of deadlocks, such as the newly identified GKC local minima. Moreover, hyperparameter tuning becomes critical, as improper parameters can not only cause infeasibility due to inter-sampling effects but also actively induce such complex local minima. Ultimately, while safety is guaranteed under our SD-NCBF framework, its pairing with CLFs should not be viewed as a substitute for a global planner, but rather as a complement acting as a real-time, geometry-aware safety filter.

\section{Conclusion and Future work}
This article proposes a novel NCBF framework based on the exact SDF between polytopes, computed via companion convex programs in \mdspace. By leveraging the geometric properties of 2D M-sums and the KKT conditions of the dist-QP and depth-LP, we derive closed-form exact SDF gradients through a unified formulation for the safe, contact, and penetration cases. We identify a new type of deadlock, revealed by the exact local rotational gradient and induced by the coupling between exact robot geometry and the nonholonomic constraints of the system.
We validate our approach in four scenarios and compare it with existing methods, highlighting non-conservative maneuvering, collision recovery, and multi-obstacle navigation. Future work will study strategies for mitigating GKC local minima and extend the proposed method to multi-robot and high-relative-degree systems.

\section*{Acknowledgments}
The authors thank David Mount at the University of Maryland for his valuable insights into Minkowski operations, and Akshay Thirugnanam at the University of California, Berkeley, for sharing and clarifying the baseline code.

\bibliographystyle{ieeetr}
\bibliography{references}  
\input{biography}

\renewcommand{\theequation}{A.\arabic{equation}}
\setcounter{equation}{0}

\appendices
\refstepcounter{section}
\label{app:pf-ift-partialL}
\section*{Appendix~\thesection: Proof of Proposition~\ref{prop: gradient-equivalence}}
\label{app: pf-ift-partialL}
\input{proof_prop1}

\end{document}

%% file: preamble.tex
\usepackage{amsmath,amsfonts}
\usepackage{algorithmic}
\usepackage{array}
\usepackage[caption=false,font=normalsize,labelfont=sf,textfont=sf]{subfig}
\usepackage{textcomp}
\usepackage{stfloats}
\usepackage{url}
\usepackage{verbatim}
\usepackage{graphicx}
\def\BibTeX{{\rm B\kern-.05em{\sc i\kern-.025em b}\kern-.08em
    T\kern-.1667em\lower.7ex\hbox{E}\kern-.125emX}}
\usepackage{balance}

\newcommand{\cmark}{\raisebox{0.8ex}{\scalebox{0.8}{$\sqrt{}$}}}
\newcommand{\xmark}{\scalebox{0.85}[1]{$\times$}}
\usepackage[caption=false,font=normalsize,labelfont=sf,textfont=sf]{subfig}
\usepackage{stfloats}
\usepackage{bm}

\usepackage{amsthm,amssymb}
\usepackage{colortbl}
\usepackage[dvipsnames]{xcolor}
\usepackage{etoolbox}
\usepackage{lipsum} 

\usepackage[dvips,frame,rotate,import]{xy} 

\makeatletter
\patchcmd{\thebibliography}{\section*{\color{black}}}{\section*{References}}{}{}
\makeatother

\makeatletter
\let\NAT@parse\undefined
\makeatother
\usepackage{hyperref}

\definecolor{light-gray}{gray}{0.93}
\definecolor{semilight-gray}{gray}{0.8}

\makeatletter
\newcommand{\figcaption}[1]{\def\@captype{figure}\caption{#1}}
\makeatother

\theoremstyle{definition}
\newtheorem{thm}{Theorem}
\newtheorem{lemma}{Lemma}
\newtheorem{prob}{Problem}
\newtheorem{coro}{Corollary}
\newtheorem{prop}{Proposition}

\newtheoremstyle{myRemark}
  {\topsep}             
  {\topsep}             
  {\itshape}            
  {}                    
  {\bfseries}   
  {.}                   
  {5pt plus 1pt minus 1pt} 
  {}                    
\theoremstyle{myRemark}
\newtheorem{remark}{Remark}

\newtheoremstyle{myDef}
  {\topsep}             
  {\topsep}             
  {\normalfont}         
  {}                    
  {\bfseries\itshape}   
  {.}                   
  {5pt plus 1pt minus 1pt} 
  {\thmname{#1}\thmnumber{ #2}\normalfont\thmnote{ \textit{(#3)}}}
\theoremstyle{myDef}
\newtheorem{defi}{Definition}

\newcommand{\mdspace}{$\mathcal{M}_\mathcal{D}$-space}
\newcommand{\wspace}{$\mathcal{W}$-space}

\newcommand{\cobs}{\mathcal{O}^c}
\newcommand{\Aci}{\mathbf{A}_\mathcal{I}^c}
\newcommand{\AciT}{\left(\Aci\right)^\top}
\newcommand{\bci}{\mathbf{b}_\mathcal{I}^c}
\newcommand{\aci}{\mathbf{a}_\mathcal{I}^c}
\newcommand{\aciT}{\left(\aci\right)^\top}
\newcommand{\bcci}{b_\mathcal{I}^c}
\newcommand{\lambdai}
{\boldsymbol{\lambda}_{\mathcal{I}}^*}
\newcommand{\lambdaT}{\left(\lambdai\right)^\top}

%% file: table_v1.tex
\resizebox{\textwidth}{!}{
\begin{tabular}{|l|cccccc|}
\hline
    \textbf{References}
    & \textbf{Metric}
    & \textbf{Explicitness} 
    & \begin{tabular}{@{}c@{}}\textbf{Exact} \\[-0.2em] \textbf{Geometry}\end{tabular} 
    & \begin{tabular}{@{}c@{}}\textbf{Exact Analytical} \\[-0.2em] \textbf{Gradient}\end{tabular} 
    & \begin{tabular}{@{}c@{}}\textbf{Obstacle} \\[-0.2em] \textbf{Defined in}\end{tabular}
    & \begin{tabular}{@{}c@{}}\textbf{Collision Recovery} \\[-0.2em] \textbf{Demonstrated}\end{tabular} \\
\hline
\hline  
    Thirugnanam et al. \cite{Thirugnanam.etal.ACC22, Thirugnanam.etal.arXiv23} 
    & MDF
    & Implicit
    & \cmark
    & \xmark\,\,(lower bound)
    & \wspace
    & \xmark\\
\hline
    Singletary et al. \cite{Singletary.etal.RAL22} 
    & SDF 
    & Implicit
    & \cmark
    & \xmark\,\,(linearization)
    & \wspace 
    & \xmark\\
\hline
    Peng et al. \cite{Peng.etal.ICRA23} 
    & MDF 
    & Explicit 
    & \xmark
    & \cmark 
    & \wspace 
    & \xmark\\
\hline
    Dai \& Wei et al. \cite{Dai.etal.RAL23, Wei.etal.TCST26}
    & \begin{tabular}{@{}c@{}}Scaling \\[-0.2em] factor\end{tabular} 
    & Implicit
    & \xmark
    & \xmark 
    & \wspace
    & \xmark\\
\hline
    Wu et al. \cite{Wu.etal.TCB25} 
    & SDF
    & Explicit
    & \xmark
    & \cmark
    & \mdspace
    & \xmark \\
\hline
    Molnar \cite{Molnar.CCTA25} 
    & SDF
    & Explicit
    & \xmark
    & \cmark
    & \wspace
    & \xmark \\
\hline
    Usevitch \& Sahleen \cite{Usevitch.Sahleen.ACC25} 
    & MDF
    & Implicit
    & \xmark
    & \xmark
    & \wspace
    & \xmark \\
\hline
    \rowcolor{light-gray} Ours 
    & SDF 
    & Implicit 
    & \cmark
    & \cmark
    & \mdspace 
    & \cmark \\
\hline
   \multicolumn{7}{p{5.5in}}{} \\[-0.75em]
   \multicolumn{7}{p{5.8in}}
   {\textit{Implicit} indicates that CBFs are defined via optimization \cite{Thirugnanam.etal.ACC22,Wei.etal.TCST26,Usevitch.Sahleen.ACC25}, or computed algorithmically \cite{Singletary.etal.RAL22}. \textit{Analytical Gradient} indicates whether exact closed-form expressions are used (as opposed to approximations \cite{Thirugnanam.etal.ACC22, Singletary.etal.RAL22} or autodiff packages \cite{Dai.etal.RAL23, Wei.etal.TCST26}). \textit{Collision recovery} refers to the ability to recover from an initial collision with an obstacle; this property is theoretically enabled by SDF-based CBFs \cite{Wu.etal.TCB25, Singletary.etal.RAL22} but has not been demonstrated in prior work.}\\
\end{tabular}}

%% file: fig-flow-chart.tex
\begin{figure}[hbt!]
    \centering\includegraphics[width=0.9\linewidth]{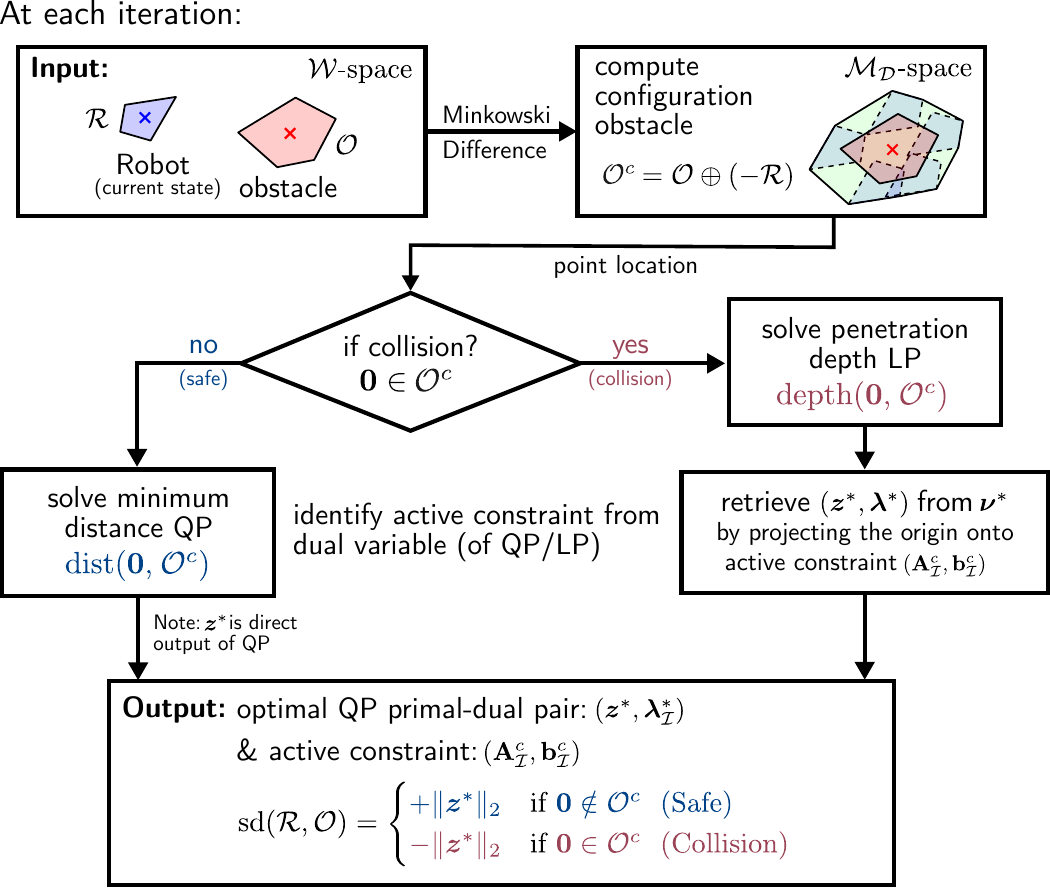}
    \caption{Computational pipeline of the proposed SDF. The framework unifies the dist-QP \eqref{eq: dist_md} and the depth-LP \eqref{eq: pd_md} by leveraging their geometric and algebraic equivalence for computing the SDF and its gradient.}
    \label{fig: flow-cvxopt}
\end{figure}

%% file: fig-nondiff.tex
\begin{figure*}[hbt!]
    \includegraphics[width=\textwidth]{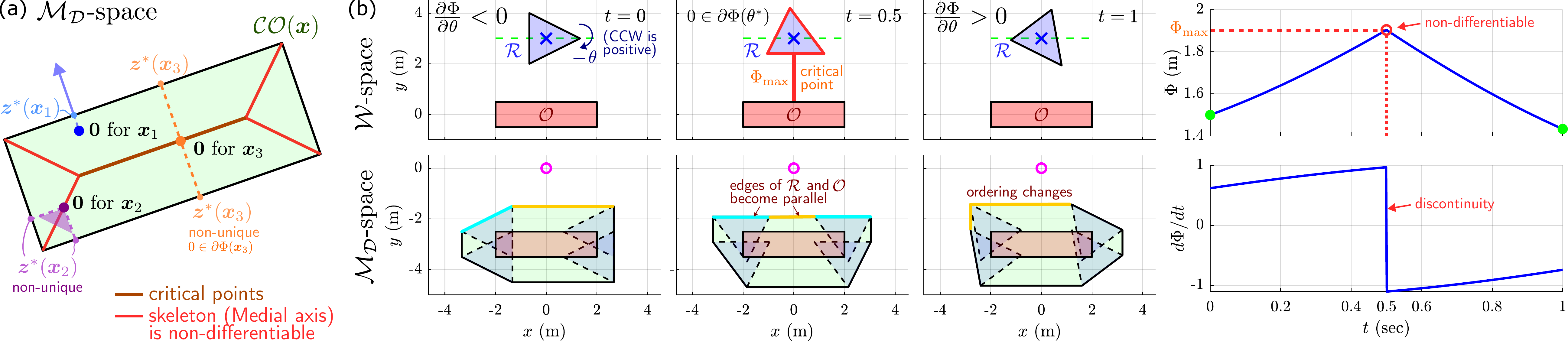}    
    \caption{Illustration of non-differentiable 
    and critical points. \textbf{(a) Translation:} The skeleton of the CO (red lines, including vertices) defines the set of non-differentiable states where $\bm{z}^*$ is non-unique. The set of critical points (brown) is a strict subset of the skeleton. The colored dots indicate the relative locations of the origin $\bm{0}$ within the \mdspace~for various robot states ($\bm{x}_1,\bm{x}_2,\bm{x}_3$) under a fixed orientation. The purple shaded region illustrates the convex hull of two associated active normals for $\bm{x}_2$. Visualization adapted from  \cite{Fernandez.Merchante.CAD24}. \textbf{(b) Rotation:} The edge-edge configuration is non-differentiable, and the corresponding orientation is a critical point, i.e., a (local) maximizer of $\Phi(\theta)$. Note that, for the illustrated CW rotation with $\dot{\theta} < 0$, the time derivative $d\Phi/dt$ has the opposite sign from the gradient $\partial\Phi/\partial\theta$ via the chain rule.} 
    \label{fig: non-diff}
\end{figure*}

%% file: fig-msum-demo.tex
\begin{figure}[hbt!]
\centering\includegraphics[width=0.95\linewidth]{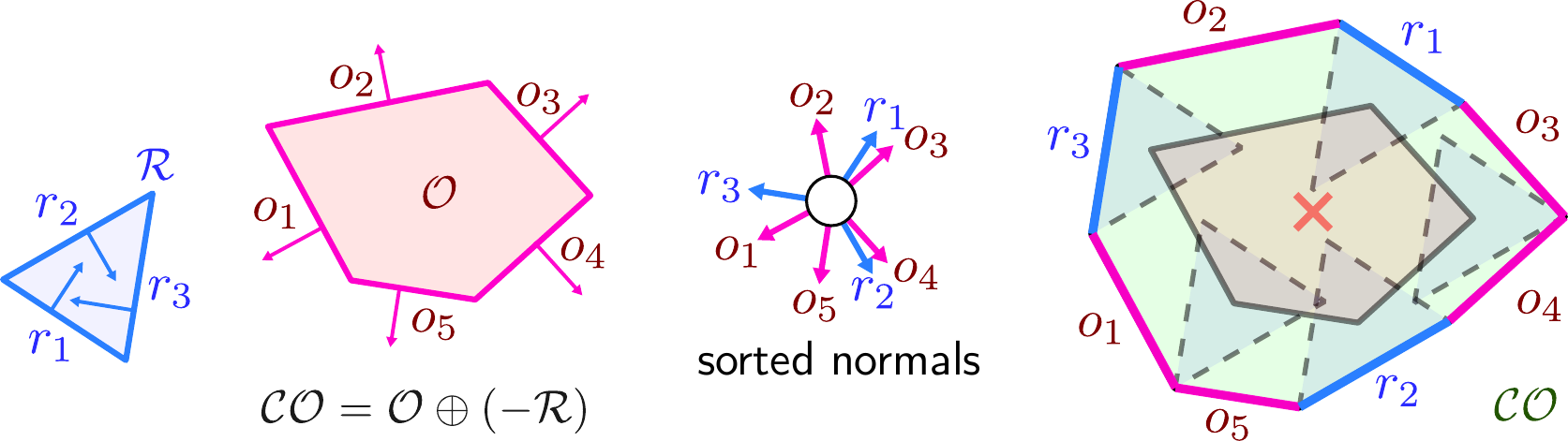}
    \caption{The order of edges of the $\mathcal{CO}$ is directly determined by collecting the inward edge normals of $\mathcal{R}$ (i.e., outward normals of  $-\mathcal{R}$)  together with the outward edge normals of $\mathcal{O}$, and sorting them around $\mathbb{S}^1$ \cite{LaValle.Planning.06}.}
\label{fig: msum-demo}
\end{figure}

%% file: fig-opt-4cases-v.tex
\begin{figure}[hbt!]
    \centering\includegraphics[width=0.95\linewidth]{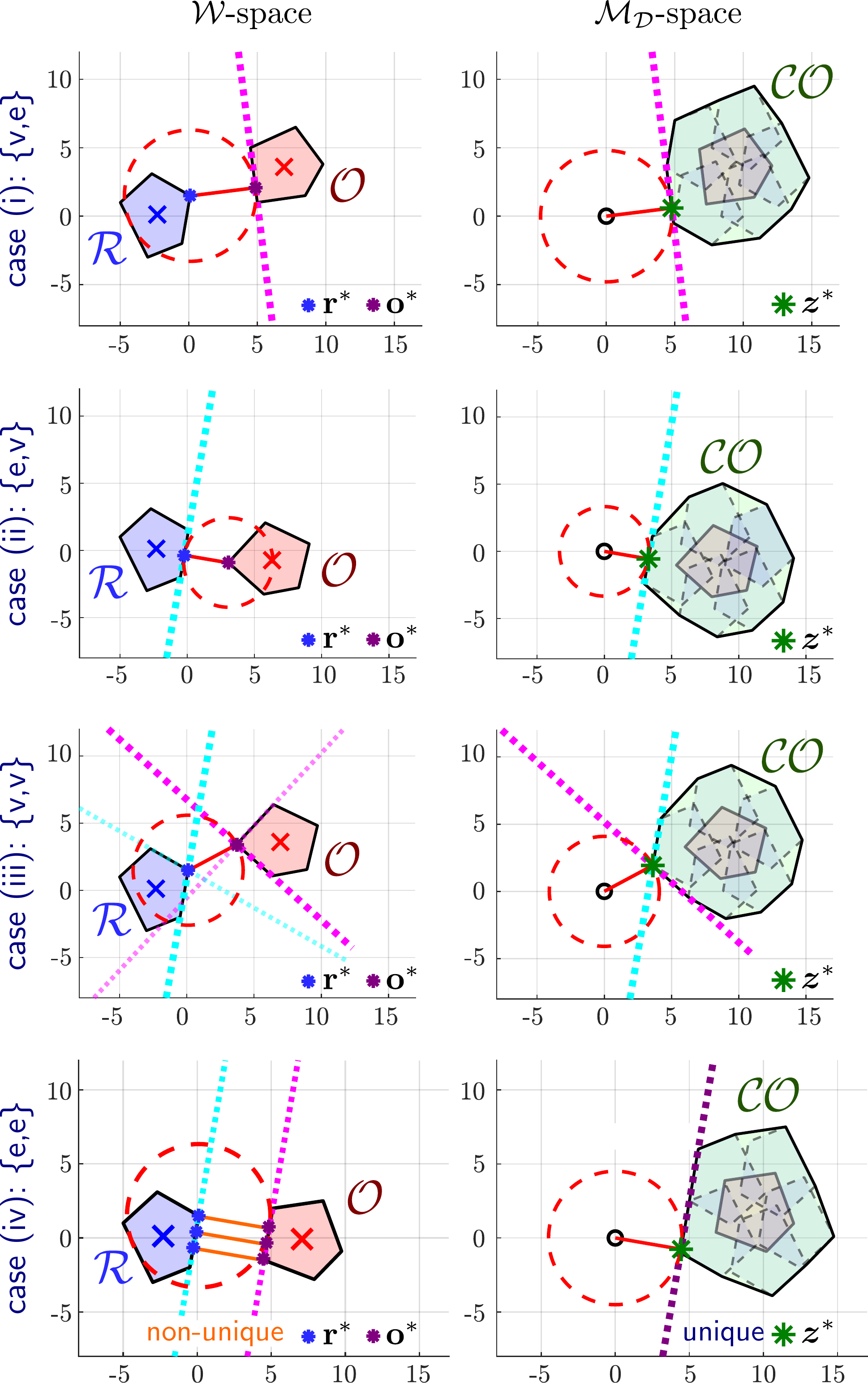}
    \caption{Mapping between the closest feature pair $(\mathbf{r}^*, \mathbf{o}^*)$ of $\mathcal{R}$ and $\mathcal{O}$ in the \wspace~ \textbf{(left column)} and the witness point $\bm{z}^*$ in \mdspace~\textbf{(right column)}. The color of constraints in \mdspace~specifies their origin: blue for the robot, magenta for the obstacle, and purple implies their combination.}
    \label{fig: cvxopt-4cases}
\end{figure}

%% file: fig-CO-rot-deri.tex
\begin{figure}[hbt!]
    \centering\includegraphics[width=0.98\linewidth]{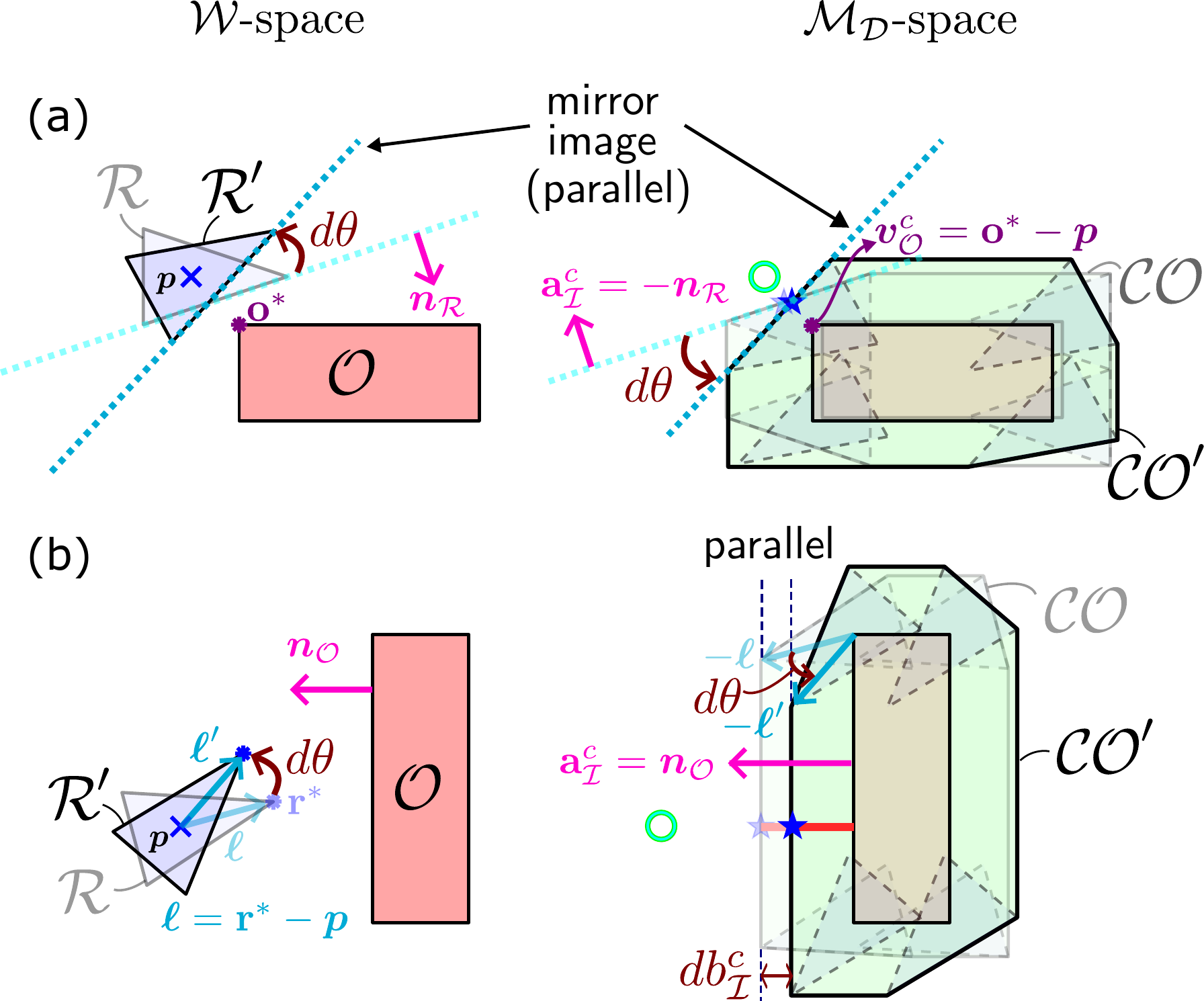}
    \caption{Rotational sensitivity depends on geometric source. \textbf{(a)} A robot-sourced edge rotates with respect to $\bm{v}_{\mathcal{O}}^c=\mathbf{o}^*-\bm{p}$; \textbf{(b)} an obstacle-sourced edge only translates with the robot witness vector $\bm{\ell} = \mathbf{r}^*- \bm{p}$.
    }
    \label{fig: rotation-deri}
\end{figure}

%% file: biography.tex
\begin{IEEEbiography}[{\includegraphics[width=1in,height=1.25in,clip,keepaspectratio]{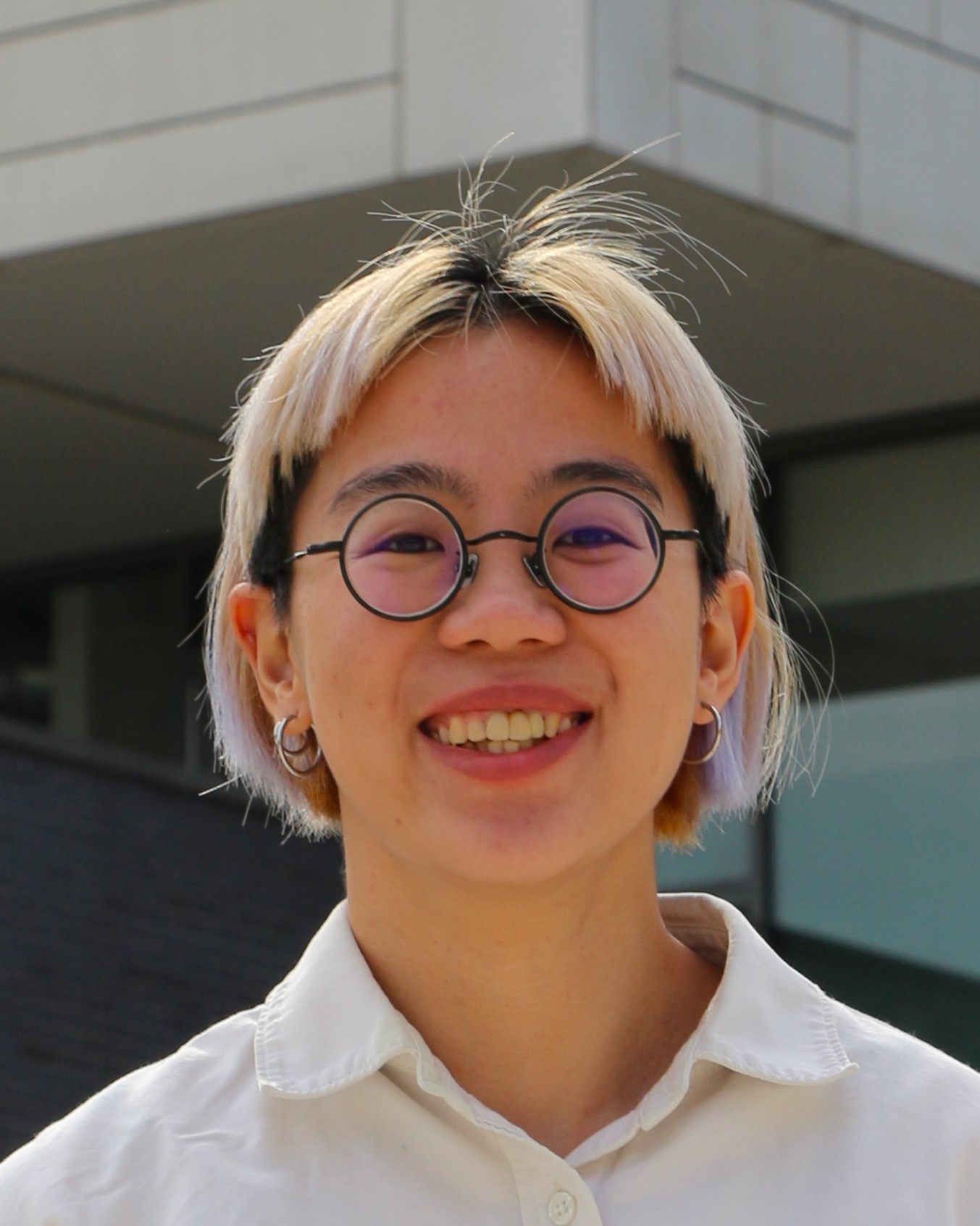}}]{Yi-Hsuan Chen} (Student Member, IEEE) 
received her B.Sc. degree in aeronautics and astronautics from National Cheng Kung University, Tainan, Taiwan, in 2019, and the M.Sc. degree in Mechanical Engineering from King Abdullah University of Science and Technology, Thuwal, Saudi Arabia, in 2022. She is currently pursuing the Ph.D. degree in Aerospace Engineering at University of Maryland, College Park, MD, USA. Her research interests include control theory, motion planning, and optimization.
\end{IEEEbiography}

\begin{IEEEbiography}[{\includegraphics[width=1in,height=1.25in,clip,keepaspectratio]{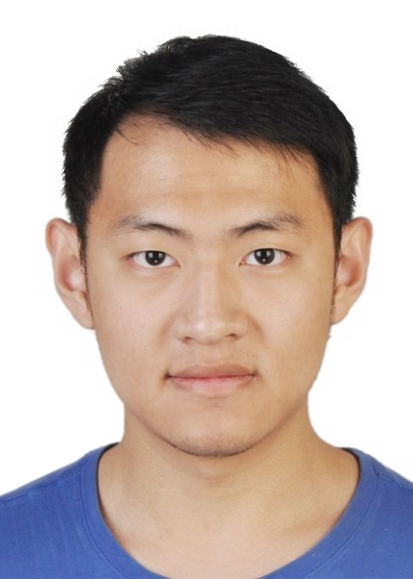}}]{Shuo Liu} (Student Member, IEEE) received his B.Eng. degree in Mechanical Engineering from Chongqing University, Chongqing, China, in 2018, and his M.Sc. degree in Mechanical Engineering from Columbia University, New York, NY, USA, in 2020. He is currently a Ph.D. candidate in Mechanical Engineering at Boston University, Boston, USA. His research interests include optimization, nonlinear control, deep learning, and robotics.
\end{IEEEbiography}

\begin{IEEEbiography}[{\includegraphics[width=1in,height=1.25in,clip,keepaspectratio]{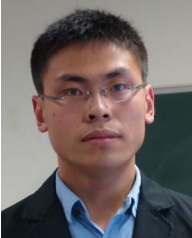}}]{Wei Xiao} (Member, IEEE) received the B.Sc. degree from the University of Science and Technology Beijing, China, in 2013, the M.Sc. degree from the Chinese Academy of Sciences (Institute of Automation), China, in 2016, and the Ph.D. degree from the Boston University, Brookline, MA, USA in 2021. He is currently a Nanyang assistant professor at the School of Electrical and Electronic Engineering, Nanyang Technological University, Singapore, a PI at M3S, SMART, and a research affiliate with MIT CSAIL. He was an assistant professor at Robotics Engineering Department at WPI and a postdoctoral associate at Massachusetts Institute of Technology. His research interests include control theory and machine learning, with particular emphasis on robotics and traffic control. He received an Outstanding Student Paper Award at the 2020 IEEE Conference on Decision and Control.
\end{IEEEbiography}

\begin{IEEEbiography}[{\includegraphics[width=1in,height=1.25in,clip,keepaspectratio]{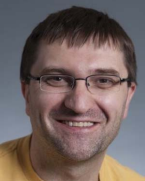}}]{Calin Belta} (Fellow, IEEE) received the B.Sc. and M.Sc. degrees from the Technical University of Iasi, Iasi, Romania, in 1995 and 1997, respectively, and the M.Sc. and Ph.D. degrees from the University of Pennsylvania, Philadelphia, PA, USA, in 2001 and 2003. He is currently the Brendan Iribe Endowed Professor of electrical and computer engineering and computer science at the University of Maryland, College Park. His research focuses on dynamics and control theory, with particular emphasis on cyberphysical systems, formal methods, and applications to robotics and systems biology. Notable awards include the 2005 National Science Foundation CAREER Award, the 2008 AFOSR Young Investigator Award, and the 2017 IEEE TCNS Outstanding Paper Award. He is a Distinguished Lecturer of the IEEE CSS.
\end{IEEEbiography}

\begin{IEEEbiography}[{\includegraphics[width=1in,height=1.25in,clip,keepaspectratio]{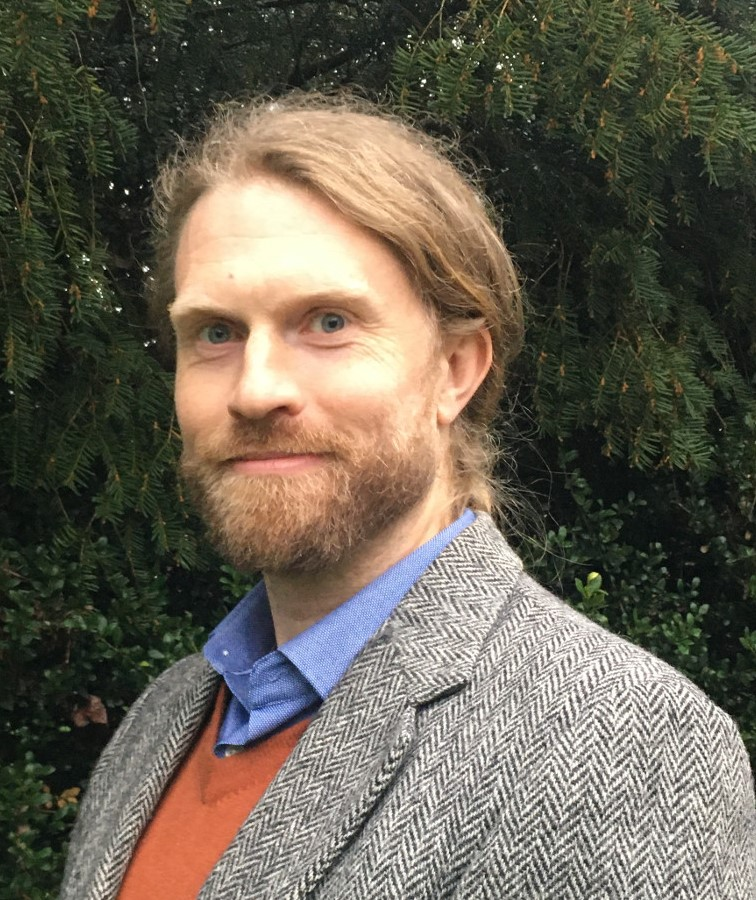}}]{Michael Otte} (Senior Member, IEEE) received the B.S. degrees in aeronautical engineering and computer science from Clarkson University, Potsdam, NY, USA, in 2005, and the M.S. and Ph.D. degrees in computer science from the University of Colorado Boulder, Boulder, CO, USA, in 2007 and 2011, respectively. From 2011 to 2014, he was a Postdoctoral Associate with the Massachusetts Institute of Technology, Cambridge, MA, USA. 
From 2014 to 2015, he was a Visiting Scholar with the U.S. Air Force Research Laboratory, Wright-Patterson Air Force Base, OH, USA. From 2016 to 2018, he was a National Research Council RAP Postdoctoral Associate with the U.S. Naval Research Laboratory, Washington, DC, USA. He has been with the Department of Aerospace Engineering, University of Maryland, College Park, MD, USA, since 2018, where he is currently an Associate Professor.  His research interests include autonomous robotics, motion planning, and multi-agent systems. 

Dr. Otte is an Associate Fellow of AIAA. He has been an Associate Editor of IEEE International Conference on Robotics and Automation and IEEE/RSJ International Conference on Intelligent Robots and Systems since 2020, for the International Journal of Robotics Research since 2023, and for AIAA Journal of Aerospace Information Systems since 2026. 
\end{IEEEbiography}

%% file: proof_prop1.tex
\label{app: pf-ift-partialL}
Consider the safe case $\mathcal{R}(\bm{x})\cap \mathcal{O}= \varnothing\Leftrightarrow \bm{0}\notin\cobs(\bm{x})$, so $\mathrm{sgn}(\mathrm{sd}(\mathcal{R}, \mathcal{O})) = 1$ and $\|\bm{z}^*(\bm{x})\|\neq0$.
We show that the IFT-based \eqref{eq: dhdx_ift} and Lagrangian-derivative-based \eqref{eq: dhdx-dL} gradient expressions are equivalent, that is:
\begin{align}
    \frac{\partial h}{\partial\bm{x}}=\frac{(\bm{z}^*)^\top}{\|\bm{z}^*\|_2}\frac{\partial \bm{z}^*}{\partial \bm{x}} = \frac{\lambdaT}{2\|\bm{z}^*\|_2}\left(\frac{\partial \Aci}{\partial \bm{x}}\bm{z}^*-\frac{\partial \bci}{\partial \bm{x}}\right).
\label{eq: pf-prop1-goal}
\end{align}
From Theorem~\ref{thm: diffopt-ift}, the sensitivity of the optimal variables $\bm{\xi}_\mathcal{I}^* = (\bm{z}^*, \lambdai)$ satisfies $\frac{\partial \bm{\xi}_\mathcal{I}^*}{\partial \bm{x}}=-\left[\frac{\partial \Gamma(\bm{\xi}_\mathcal{I}^*)}{\partial \bm{\xi}}\right]^{-1}\frac{\partial \Gamma(\bm{\xi}_\mathcal{I}^*)}{\partial \bm{x}}$, where
\begin{subequations}
\begin{align}
    &\frac{\partial \Gamma(\bm{\xi}_{\mathcal{I}}^*)}{\partial \bm{\xi}}
    =
    \begin{bmatrix}
    2\mathbf{I} & \left(\Aci\right)^{\top} \\D(\lambdai)\Aci & \bm{0}
    \end{bmatrix}:=\mathbf{M} , \label{eq: appen-dGammadXi} \\[0.2em]
    &\frac{\partial \Gamma(\bm{\xi}_{\mathcal{I}}^*)}{\partial \bm{x}} 
    =
    \begin{bmatrix}
    \partial_{\bm{x}} \AciT \lambdai\\[0.2em]
    D(\lambdai) \left(\frac{\partial \Aci}{\partial \bm{x}}\bm{z}^*-\frac{\partial \bci}{\partial \bm{x}}\right)
    \end{bmatrix}
    :=\begin{bmatrix}
    \mathbf{T}_1\\\mathbf{T}_2
    \end{bmatrix},
    \label{eq: appen-dGammadx}
\end{align}
\end{subequations}
Let $\mathbf{M}^{-1}=\begin{bmatrix}
\mathbf{M}_{11}^-  & \mathbf{M}_{12}^-\\ \mathbf{M}_{21}^- & \mathbf{M}_{22}^-
\end{bmatrix}$. Using the Schur complement with $\mathbf{P}:=\Aci \AciT\in\mathbb{R}^{|\mathcal{I}|\times |\mathcal{I}|}$ yields
\begin{subequations}
\begin{align}
    & \mathbf{M}_{11}^- =  \frac{1}{2}\left[\mathbf{I}_2-\AciT\mathbf{P}^{-1}\Aci\right] ,\label{eq: appen-M11}\\[0.2em]
    &\mathbf{M}_{12}^-=\AciT\mathbf{P}^{-1}\left[ D(\lambdai)\right]^{-1}. \label{eq: appen-M12}
\end{align}
\end{subequations}
We have $\frac{\partial \bm{z}^*}{\partial \bm{x}}=-(\mathbf{M}_{11}^-\mathbf{T}_1+\mathbf{M}_{12}^-\mathbf{T}_2)$. By taking the transpose of the stationarity condition, we obtain $(\bm{z}^*)^\top=-\frac{1}{2}\lambdaT\Aci$. This gives
\begin{align*}
    &(\bm{z}^*)^\top\mathbf{M}_{11}^-=-\frac{1}{4}\left[  \lambdaT\Aci-\lambdaT \mathbf{P}\mathbf{P}^{-1}
    \Aci \right] = \bm{0},\\
&(\bm{z}^*)^\top\mathbf{M}_{12}^-\mathbf{T}_2
=-\frac{1}{2}\lambdaT
\left(\frac{\partial \Aci}{\partial\bm{x}}\bm{z}^*-\frac{\partial \bci}{\partial\bm{x}}  \right).
\end{align*}
Hence, we have
\begin{align*}
    \frac{\partial h}{\partial \bm{x}} &=\frac{1}{\|\bm{z}^*\|_2}\left[(\bm{z}^*)^\top\frac{\partial \bm{z}^*}{\partial \bm{x}}\right]=\frac{1}{\|\bm{z}^*\|_2}\left[-(\bm{z}^*)^\top\mathbf{M}_{12}^-\mathbf{T}_2\right]
    \\
    &=\frac{\lambdaT}{2\|\bm{z}^*\|_2}\left(\frac{\partial \Aci}{\partial\bm{x}}\bm{z}^*-\frac{\partial \bci}{\partial\bm{x}}  \right),  
    \label{eq: appen-dhdx-final}
\end{align*}
which completes the proof.